\documentclass[onefignum,onetabnum]{siamonline250211}
\usepackage{amsfonts}
\usepackage{graphicx}
\usepackage{epstopdf}
\usepackage{algorithmic}
\ifpdf
  \DeclareGraphicsExtensions{.eps,.pdf,.png,.jpg}
\else
  \DeclareGraphicsExtensions{.eps}
\fi
\usepackage{enumitem}
\setlist[enumerate]{leftmargin=.5in}
\setlist[itemize]{leftmargin=.5in}

\newsiamremark{remark}{Remark}
\newsiamremark{hypothesis}{Hypothesis}
\crefname{hypothesis}{Hypothesis}{Hypotheses}
\newsiamthm{claim}{Claim}
\newsiamremark{fact}{Fact}
\crefname{fact}{Fact}{Facts}
\headers{Inclusive KL Gradient Flows}{J.-J. Zhu}
\title{Inclusive KL Gradient Flows: Otto-Wasserstein,\\ Fisher-Rao-Gaussian, and Local-Estimator Dynamics\thanks{\funding{The author acknowledges the support from the Deutsche Forschungsgemeinschaft (DFG, German Research Foundation) as part of the priority programme ``Theoretical Foundations of Deep Learning'' (project number: 543963649).}}}
\author{Jia-Jie Zhu\thanks{Department of Mathematics, KTH Royal Institute of Technology, Lindstedtsvägen 25, 114 28, Stockholm, Sweden (\email{jiajie@kth.se}).}}
\usepackage{amsopn}

\newcommand{\WGF}{\mathrm{WGF}}
\newcommand{\WGFloc}{\ensuremath{\WGF_{\mathrm{loc}}}}

\newcommand{\EEE}{\color{black}}

\usepackage{preamble-JZ}

\usepackage{hyperref}       %
\usepackage{booktabs}       %
\usepackage{amsfonts}       %
\usepackage{nicefrac}       %
\usepackage{microtype}      %
\usepackage{xcolor}         %
\usepackage{graphicx}

\usepackage{amsmath,xcolor,mathrsfs}
\usepackage{hyperref}
\usepackage{cleveref}       %
\usepackage{bbm}
\usepackage{bm}
\usepackage{xspace}
\usepackage{wrapfig}
\usepackage{float}
\usepackage{algorithm}
\usepackage{algorithmic}

\newcommand{\Mplus}{\mathcal{M}^+}
\newcommand{\rkhs}{\mathcal{H}}

\usepackage{multirow} %
\usepackage[normalem]{ulem} %

\let\citep\cite
\let\citet\cite
\begin{document}
\maketitle
\begin{abstract}
Otto's Wasserstein gradient flow of the inclusive (forward) Kullback--Leibler (KL) divergence offers a principled framework for analyzing statistical inference algorithms, yet algorithms targeting the exclusive (reverse) KL divergence are rarely studied with such tools. We establish a unified gradient-flow and PDF framework for inclusive KL inference. We show that maximum mean discrepancy minimization can be viewed as inclusive KL inference with an approximate gradient estimator, and we develop the Fisher--Rao and Wasserstein--Fisher--Rao gradient flows that directly target the inclusive KL divergence. Restricting these flows to the manifold of Gaussian distributions yields explicit gradient-flow ODEs, providing a foundation for Gaussian variational inference.
Building on this viewpoint, we further introduce a local-estimator Wasserstein gradient flow whose velocity is obtained by local nonparametric regression---free of density-ratio evaluation or kernel gradients---improving the algorithmic performance over the MMD-based particle method.
\end{abstract}

\section{Introduction}
\label{sec:intro}
Many inference and learning problems can be cast into the framework of
minimizing the Kullback-Leibler (KL) divergence
\begin{align}
    \min_{\mu \in A\subset \mathcal{P}} \DKL(\mu | \pi)
      .
      \label{eq:exclusiveKL}
\end{align}
Here, $\mathcal{P}$ denotes the space of probability distributions.
The functional $\DKL(\mu | \pi)$
is also known as the reverse KL divergence
between $\mu$ and $\pi$.
This variational problem forms the foundation of modern Bayesian inference \citep{zellner1988optimal}.
For example, suppose we have a model $p(\mathrm{Data}|\theta)$ and a prior $p(\theta)$, our goal is to infer the posterior $\pi(\theta) := p(\theta| \mathrm{Data})$.
If we further restrict the feasible set $A$ in \eqref{eq:exclusiveKL} to be the so-called variational family, e.g., the set of all Gaussian distributions, we obtain variational inference~\citep{jordanIntroductionVariationalMethods1999,wainwrightGraphicalModelsExponential2008,blei_variational_2017}.
Albeit much less popular,
there also exists the inference paradigm that minimizes
the forward KL divergence,
\begin{align}
    \min_{\mu \in A\subset \mathcal{P}} \mathrm{D}_\mathrm{KL}( \pi | \mu )
      .
      \label{eq:inclusiveKL}
\end{align}
The functional $\KL(\mu | \pi)$ in \eqref{eq:inclusiveKL}
is also known as the exclusive KL divergence, 
due to its well-known property commonly referred to as mode-seeking and zero-avoiding when the set $A$ is chosen as the Gaussian family.
For the same reason, 
$\mathrm{D}_\mathrm{KL}( \pi | \mu )$ is also known as the inclusive KL, due to its well-known property commonly referred to as mode-covering and zero-forcing.
\footnote{To avoid the potential confusion caused by different interpretation of the term forward and reverse, in this paper, we will use the term inclusive and exclusive instead.}
For example, algorithms such as expectation propagation~\citep{minkaExpectationPropagationApproximate2013}, \citep[Section~10.7]{bishop2006pattern}
can be viewed as solving \eqref{eq:inclusiveKL}.
Many researchers
such as
\citet{naessethMarkovianScoreClimbing2020,jerfelVariationalRefinementImportance2021,mcnamaraSequentialMonteCarlo2024,zhangTransportScoreClimbing2022}
have argued that the solution of \eqref{eq:inclusiveKL}, if available, offers statistical advantages over \eqref{eq:exclusiveKL}.
We also refer to the discussion in \citep{dhakaChallengesOpportunitiesHighdimensional2021}
about the behavior of inclusive KL variational inference in moderate-to-high dimensions.
However, most of the existing algorithms targeting inclusive KL minimization 
require adhoc procedures and do not have sound mathematical analysis as the backbone; see also our later discussion around \eqref{eq:vanilla-wasserstein-rkl-gfe}.
\begin{figure}[htbp]
    \centering
    \begin{minipage}[t]{0.32\textwidth}
        \vspace{0pt}
        \centering
        \includegraphics[width=\linewidth]{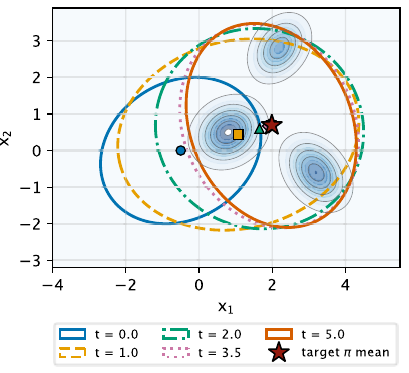}\\
        {\footnotesize(a)}
    \end{minipage}%
    \hfill%
    \begin{minipage}[t]{0.32\textwidth}
        \vspace{0pt}
        \centering
        \includegraphics[width=\linewidth]{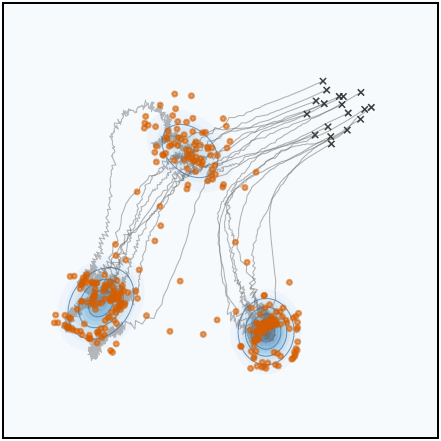}\\
        {\footnotesize(b)}
    \end{minipage}%
    \hfill%
    \begin{minipage}[t]{0.32\textwidth}
        \vspace{0pt}
        \centering
        \includegraphics[width=\linewidth]{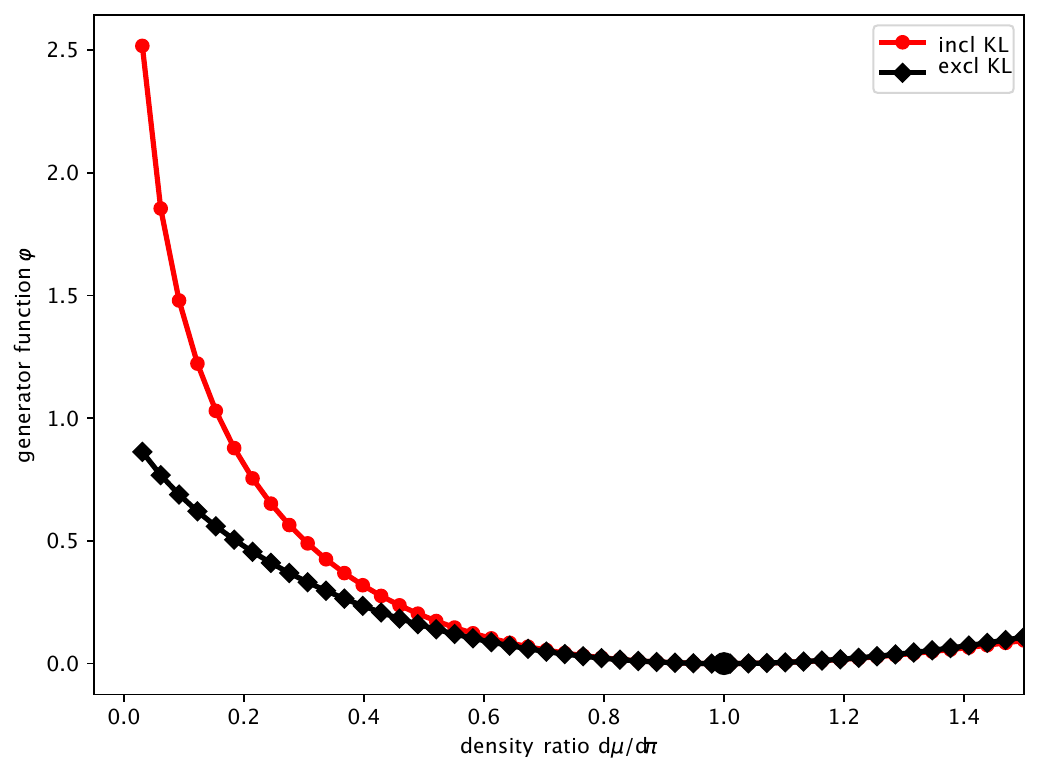}\\
        {\footnotesize(c)}
    \end{minipage}
    \caption{
    \emph{(a)} Evolution of the Fisher--Rao Gaussian gradient flow of the inclusive KL functional toward a mixture of three Gaussians; see Section~\ref{sec:gaussian-manifold}.
    \emph{(b)} Preview of the proposed local-estimator Wasserstein gradient flow of the \emph{inclusive} KL ($\mathrm{iKL}\text{-}\WGFloc$): initialization ($\times$), particles sweep in to cover all three modes of a Gaussian mixture; see Section~\ref{sec:local_estimate}.
    \emph{(c)} The generator $\varphi$ underlying the exclusive and inclusive KL divergences.}
    \label{fig:intro-overview}%
    \label{fig:fr-ode-intro}\label{fig:forward_reverse_kl}
\end{figure}
In comparison, 
there has been significant technical developments for the exclusive KL minimization~\eqref{eq:exclusiveKL} recently.
This is mainly due to the injection of rigorous theoretical foundation from
analysis of (PDE) gradient flows~\citep{otto1996double,ottoGeometryDissipativeEvolution2001,ambrosio2008gradient,peletier_variational_2014,mielke2023introduction}
and
statistical optimal transport~\citep{chewiStatisticalOptimalTransport2024,peyre2018computational,panaretosStatisticalAspectsWasserstein2019}.
Inference and sampling algorithms based on \eqref{eq:exclusiveKL} can now be studied under a unified framework and on the rigor level of applied analysis.
Can we use such principled theory
to study inclusive KL minimization?
This paper answers this question affirmatively.
Beyond placing existing methods on a rigorous analytical footing, our development also yields new constructions---explicit ODEs on the Gaussian manifold and a local-estimator Wasserstein gradient flow that exhibits competitive empirical performance.
See Figure~\ref{fig:fr-ode-intro} (a,b) for a visual illustration of an inclusive-KL-driven gradient flows proposed in this paper.

\paragraph*{Overview of main results}
Our starting point is the
Otto's Wasserstein gradient flow of the inclusive-KL divergence, generating the flow equation~\eqref{eq:vanilla-wasserstein-rkl-gfe}, a logarithmic-diffusion equation $\dot\mu=\DIV\!\big(\mu\,\nabla(1-\tfrac{\dd\pi}{\dd\mu})\big)$. 
We make the following two observations:
Replacing the force $1-\tfrac{\dd\pi}{\dd\mu}$ by its kernel-smoothed (\emph{force-kernelized} as in Definition~\ref{def:force-kernelization}) counterpart, the flow becomes \emph{exactly} the Wasserstein gradient flow of the squared MMD (Theorem~\ref{thm:equivalence-of-gradient-flow-equations}).
Thus MMD--Wasserstein gradient flow is inclusive-KL inference with a mollified gradient estimator.

The Fisher--Rao flow admits an MMD--MMD realization along which the MMD is a Lyapunov functional, without requiring log-Sobolev-type conditions (Theorem~\ref{thm:lyapunov-MMD}):
\begin{align*}
    \tfrac{\dd}{\dd t}\,\DKL(\pi\,|\,\mu_t)\ \le\ c\,\tfrac{\dd}{\dd t}\,\mmd^2(\mu_t,\pi),
    \qquad c>0 .
\end{align*}
Combining the two geometries gives global convergence of the unbalanced Wasserstein--Fisher--Rao flow.

On the Gaussian manifold the flows reduce to ODE systems for the mean--covariance pair $(m_t,\Sigma_t)$; the Fisher--Rao--Gaussian ODE system reads (Proposition~\ref{prop:gvi-gfe}, Figure~\ref{fig:fr-ode-intro})
\begin{align*}
    \dot m_t = -(m_t-m_\pi),
    \qquad
    \dot\Sigma_t = -\big(\Sigma_t-\Sigma_\pi-(m_\pi-m_t)(m_\pi-m_t)^\top\big),
\end{align*}
with an analogous Bures--Wasserstein system~\eqref{eq:gvi-bures-wasserstein-gfe-2}, providing a foundation for Gaussian variational inference under the inclusive KL divergence
with convergence analysis (Proposition~\ref{cor:fr-ode-solution}).

Finally, building on the force-kernelization principle (Definition~\ref{def:force-kernelization}), we replace the dual force vector by a local nonparametric regression estimator, yielding the flow~\eqref{eq:local-estimator-wgf}:
\begin{align*}
    \dot\mu = \DIV\!\big(\mu\,\hat\theta_1(x)\big),
    \quad
    (\hat\theta_0,\hat\theta_1)(x)=\!\!\argmin_{\theta_0\in\bbR,\,\theta_1\in\bbR^d}\sum_{i=1}^N\left(\dFdmut(x_i)-\theta_0-\theta_1^\top(x_i-x)\right)^2 K_h(x_i-x).
\end{align*}
We refer to the resulting dynamics as the local-estimator Wasserstein gradient flow, denoted as $\WGFloc$.
The estimator $\hat\theta_1$ needs neither the density ratio $\dd\pi/\dd\mu$ nor kernel gradients, and empirically converges faster and more stably over a broad range of bandwidths than the MMD-WGF of \citet{arbel_maximum_2019}, as previewed in Figure~\ref{fig:intro-overview}(b).
This demonstrates the potential of the principled inclusive KL inference enabled by our gradient flow theory.
\EEE
\section{Preliminaries}
\paragraph*{Notation}
By default, we work on the base space $\bbR^d$.
The measures that appear in this paper are by default assumed to be absolutely continuous with respect to the Lebesgue measure.
In formal derivation,
we use measures and their density interchangeably,
\ie$\int f\cdot \mu$ means the integral w.r.t. the measure $\mu$.
We use the notation $\mathcal{P}, \Mplus$ to denote the space of probability and non-negative measures on a domain that is a closed, bounded, convex subset of $\mathbb R^d$.
For many of our results, this domain can be generalized to $\bbR^d$
and the measures can be generalized to atomic measures; see \citep{ambrosio2008gradient} for more details.
The first variation of a functional $F$ at $\mu\in\Mplus$ is defined as a function ${\frac{\delta F}{\delta\mu}[\mu] }$
such that
$\frac{\dd }{\dd \epsilon}F(\mu + \epsilon \cdot v) |_{\epsilon=0}
= \int {{\frac{\delta F}{\delta\mu}[\mu] }}(x) \dd v (x)$
for any valid perturbation in measure $v$ such that $\mu + \epsilon \cdot v\in \Mplus$ when working with gradient flows over $\Mplus$ and 
$\mu + \epsilon \cdot v\in \mathcal P$ over $\mathcal P$.
To avoid confusion, we refer to the forward KL divergence $\DKL(\pi | \mu)$ as the inclusive KL; the reverse KL divergence $\DKL(\mu | \pi )$ as the exclusive KL.
The well-posedness of PDEs such as solution uniqueness, existence, and regularity are beyond the scope of this paper.
Without further specification,
the duality pairing notation $\langle f, g \rangle$ is the $L^2$ inner product
$\displaystyle\int f(x) g(x) \dd x$.

\paragraph*{Kullback-Leibler (KL) divergence}
The $\varphi$-divergence, also known as the f-divergence~\citep{csiszar1967information}, is a class of statistical divergences that measure the difference between a pair of measures,
defined as
\begin{equation}
    D_\varphi(\mu | \pi) = \int  \varphi\left(\frac{\dd\mu}{\dd\pi}(x)\right) \dd\pi(x), \text{ if } \mu \ll \pi
    \label{eq:general-divergence}
\end{equation}
and $+\infty$ otherwise; ${\dd \mu}/{\dd \pi}$ is the Radon-Nikodym derivative.
The entropy generator function $\varphi$ is a convex function satisfying
$\varphi(1)=\varphi'(1)=0, \ \varphi''(1)=1$.
Different choices of $\varphi$ lead to various well-known divergences, such as
exclusive KL: \  $\varphi_{\textrm{KL}}(s):=s\log s-s+1$, 
inclusive KL: \  $\varphi_{\textrm{revKL}}(s) = s-1 -\log s$,
\textrm{Hellinger:}\ $\varphi_{\textrm{H}}(s) =  (\sqrt{s} - 1)^2$,
$\chi^2$:\ $\varphi_{\chi^2}(s)=\frac{1}{2}(s-1)^2$.
Note that
our definition for measures that are 
not necessarily probability measures.
As a central topic of this paper, we will focus on the KL divergence evaluated in both directions, $\DKL(\mu|\pi)$ and $\DKL(\pi|\mu)$.
By elementary calculation, we observe that the forward and inclusive KL divergences have the same first-order expansion near the equilibrium point when $\mu=\pi$.
Figure~\ref{fig:forward_reverse_kl} shows the generator $\varphi$ of the exclusive and inclusive KL divergences.
Note that
entropy generator functions $\varphi$ can be made more general by the following power entropy,
\begin{equation}
    \label{eq:power-ent}
    \phiP(s):=\frac{1}{p(p-1)}\left(s^p-ps+p-1\right),\quad p\in \mathbb{R}\setminus\{0,1\}
    .
\end{equation}
Many commonly used divergences can be recovered using different choices of $p$.
The resulting
divergence functional $\mathrm{D_{\phiP}}$ is also called the $p-$relative entropy, e.g., inclusive KL $(p=0)$, exclusive KL $(p=1)$; cf. \citep{ohtaDisplacementConvexityGeneralized2011,mielke2025hellinger}.
Note an alternative parameterization of the entropy function can also be made by using the $\alpha$-divergence \citep{amariMethodsInformationGeometry2000}, defined by
$\varphi_\alpha(s) =
    \frac{4}{1-\alpha^2}\Bigl(1 - s^{\tfrac{1+\alpha}{2}}\Bigr), \alpha \in \mathbb{R}\setminus\{\pm 1\}$.
The divergence defined in \eqref{eq:general-divergence}, \eqref{eq:power-ent}
is notably studied in the PDE literature as its Wasserstein gradient flows generate the porous medium equation.

\paragraph*{Statistical inference as Wasserstein gradient flow of KL}
An elegant perspective of (Bayesian) inference is offered by the Wasserstein gradient flow (WGF) framework of \citet{otto1996double}, which has attracted much attention from researchers in Bayesian inference; see \citep{chewiStatisticalOptimalTransport2024,trillosBayesianUpdateVariational2018}
for recent surveys.
In that framework, one can write a flow equation formally as
\begin{align}
    \dot \mu =
    - \nabla _\OT F(\mu)
    =
    -  \bbKotto(\mu) \frac{\delta F}{\delta \mu}[ \mu]
    = \DIV\left(\mu\nabla \frac{\delta F}{\delta \mu}[ \mu]\right)
    .
    \label{eq:wasserstein-gfe}
\end{align}
through
the
Wasserstein Onsager operator $\bbKotto$, which is defined as the inverse of the Riemannian metric tensor $\mathbb G_\OT$
of the Wasserstein space, \ie $\bbKotto(\rho) = \mathbb G_\OT(\rho)^{-1}$.
Here,
the shorthand $\OT$
is due to the contribution of Felix Otto as well as the abbreviation of optimal transport.
Mathematically, for the Wasserstein space,
$$\bbKotto(\rho): T^*_\rho \calM \to T_\rho \calM, \xi \mapsto -\DIV(\rho\nabla \xi),$$
where $T_\rho \calM$ is the tangent space of $\Mplus$ at $\rho$ and
$T^*_\rho \calM$ the cotangent space.
The terminology Onsager's operator is due to the works of \citet{onsagerFluctuationsIrreversibleProcesses1953,onsagerReciprocalRelationsIrreversible1931}.
From the mechanics perspective, the dual functions $\xi$ can be interpreted as the generalized thermodynamic forces~\cite{onsagerFluctuationsIrreversibleProcesses1953,mielkeNonequilibriumThermodynamicalPrinciples2017}.
We provide more background on the Wasserstein gradient flows in Appendix.

With those ingredients, we can formally define the gradient systems that generate gradient flow equations such as \eqref{eq:wasserstein-gfe}.
\begin{definition}
    [Gradient system~\citep{ottoGeometryDissipativeEvolution2001,mielke2023introduction}]
    We refer to a tuple $\left( \calM, F, \bbK \right)$
    as a gradient system. 
    It has the gradient structure
    identified by:
\begin{tightenum}
        \item a space $\calM$,
        \item an energy functional $F$,
        \item a dissipation geometry given by either:
                a distance metric defined on $\calM$,
                a Riemannian metric tensor $\mathbb G$,
                or
                a symmetric positive-definite Onsager operator $\bbK=\mathbb G^{-1}$.
\end{tightenum}
\label{def:gradient-system}
\end{definition}
Note that it is also possible to define dissipation geometry via nonlinear dissipation potential functional;
cf.
\citep{mielkeNonequilibriumThermodynamicalPrinciples2017}.

Regarding Bayesian inference, 
we choose the 
energy functional as the exclusive KL divergence as in \eqref{eq:exclusiveKL},
\ie
$F(\mu) = \DKL(\mu | \pi)$.
Through elementary calculation,
we obtain from \eqref{eq:wasserstein-gfe} the Fokker-Planck equation (FPE)
\begin{align}
    \partial_t \mu =
    \DIV \left( \mu \nabla \log\frac{\mu}{\pi}  \right)
    =
    { \Delta \mu} - \DIV \left( \mu \nabla \log \pi  \right)
    .
    \tag{FPE}
    \label{eq:fokker-planck}
  \end{align}
When we express the target as $\pi(x) = \frac1Z \exp(-V(x))$
where $Z$ is a normalization constant (partition function), the \eqref{eq:fokker-planck} is then 
$\partial_t \mu = { \Delta \mu} + \DIV \left( \mu \nabla V  \right)$.
The fact that the evolution equation does not depend on the partition function $Z$ is often argued to be one of the key advantages of the KL divergence.
Viewed as a dynamic system, the KL divergence energy functional dissipates
along \eqref{eq:fokker-planck} in the steepest descent manner.
Based on the formal definition of gradient system~\eqref{def:gradient-system}, we say that \eqref{eq:fokker-planck} has the \emph{gradient structure} that entails the following key ingredients.
\begin{align}
    \begin{cases}
        \textrm{{Space} :}&
        \text{prob. space }
        \mathcal P 
        \\
        \textrm{Energy functional} :& {F(\cdot):= \mathrm{D}_\mathrm{KL}(\cdot | \pi)}
        \\
        \textrm{Dissipation Geometry} :& { \text{Wasserstein}\ \bbKotto}
    \end{cases}
    \label{eq:gradient-structure-fkl-wgf}
\end{align}

\paragraph*{Integral operator and maximum-mean discrepancy}
Given a positive measure $\rho$ on $\bbR^d$
and a positive-definite kernel $K$,
the integral operator
$\Tkrho: L^2_{\rho} \rightarrow \rkhs$ is defined by
\begin{align}
    \Tkrho g(x):=\int K\left(x, x^{\prime}\right) g\left(x^{\prime}\right) d \rho\left(x^{\prime}\right) \ \text{ for }\ g \in L^2_{\rho}
    ,
    \label{eq:def-int-op}
\end{align}
where $\rkhs$ is the reproducing kernel Hilbert space associated with the kernel $K$.
With a slight abuse of terminology, the following 
compositional operator
$\K_\rho:= \ID \circ \Tkrho$
is also referred to as the integral operator, albeit defined for $L^2(\rho) \to L^2(\rho)$.
$\K_\rho$ is compact, (semi-)positive, self-adjoint, and nuclear; cf. \citep{steinwart2008support,heinKernelsAssociatedStructures2004,steinwartMercersTheoremGeneral2012}.
The adjoint of $\Tkrho$ is the embedding operator 
$\ID :\rkhs \to L^2(\rho)$, \ie
$\langle \ID f, g \rangle_{L^2(\rho)} = \langle f, \Tkrho g \rangle_{\rkhs}$ for all $f\in\rkhs$ and $g\in L^2(\rho)$.
When using a kernel such as the Gaussian kernel, the image $\Tkrho g$
can be regarded as a smooth approximation of $g$, which is sometimes referred to as approximation by convolution or mollification~\citep{wendland_scattered_2004}.
An assumption we will generally make throughout the paper is that
the kernel $K$ is bounded, symmetric, and satisfies the \emph{integrally strict positive-definite} (ISPD) condition~\cite{JMLR:v11:sriperumbudur10a,steinwart2008support,stewartPositiveDefiniteFunctions1976}:
$\displaystyle\int K(x,x') \dd \rho(x) \dd \rho(x') > 0$
for any non-zero signed measure $\rho$.
The purpose of this condition is to ensure that the integral operator $\Tkrho$ is strictly positive-definite.
Though we also consider kernels typically used in applied math that are not ISPD.
In this paper,
the kernel maximum mean discrepancy (MMD)
between two positive measures $\mu$ and $\nu$ is defined as
\begin{multline*}
        \mmd^2( \mu, \nu): 
        =\iint K(x, x') \dd (\mu - \nu)(x) \dd (\mu - \nu)(x')
        \\
        =
          \underbrace{- \iint K(x, x') \dd\nu(x) \dd \mu (x')}_{\text{attract}}
         +  \underbrace{\iint K(x, x') \dd\mu(x) \dd \mu (x')}_{\text{repel}}
         +\ \const
\end{multline*}
We particualrly favor the second representation above due to the inteprpetation of the opposing physical mechanisms
in the context of gradient flows.

\section{Force-kernerlization: Wasserstein gradient flows of inclusive KL}
\label{sec:gradient-flows}
Our starting point is the following Wasserstein gradient flow equation of the inclusive KL inference,
derived using \citet{ottoGeometryDissipativeEvolution2001}'s formal calculation analogous to the exclusive KL case~\eqref{eq:wasserstein-gfe}.
\begin{align}
    \dot \mu 
    =
    \DIV\left(\mu \nabla \left(1-\frac{\dd \pi}{\dd \mu}\right)\right)
    .
    \label{eq:vanilla-wasserstein-rkl-gfe}
\end{align}
This equation is also referred to as the \emph{logarithmic diffusion equation}.
The relation with the exclusive KL gradient flow can be observed by comparing this PDE with \eqref{eq:fokker-planck}.
The two generalized force functionals
agree to the first order
near the equilibrium,
\ie
$\log \frac{\dd \mu}{\dd \pi} \approx 1 - \frac{\dd \pi}{\dd \mu}$ when $\dd \pi / \dd \mu \approx 1.$
Furthermore,
the generator function of the inclusive KL has a larger slope than that of the exclusive KL.
This subtle difference will lead to different behaviors of their gradient flows.
Rewriting the right-hand side of \eqref{eq:vanilla-wasserstein-rkl-gfe},
we obtain the PDE
$\dot \mu = -\Delta \pi + \DIV (\pi \nabla \log \mu)$,
which bears similarity to \eqref{eq:fokker-planck}
but with the position of $\pi$ and $\mu$ exchanged on the right-hand side.
Intuitively, the gradient structure of
\eqref{eq:vanilla-wasserstein-rkl-gfe} is given by:
\begin{align}
    \begin{cases}
        \textrm{{Space} :}& 
        \mathcal P 
        \\
        \textrm{Energy functional} :& {F(\cdot):= \mathrm{D}_\mathrm{KL}( \pi | \cdot)}
        \\
        \textrm{Dissipation Geometry} :& { \text{Wasserstein}\ \bbKotto}
        \end{cases}
        \label{eq:gradient-structure-rkl-wgf}
\end{align}
While the gradient structure is clear,
a main obstacle to implement the
gradient flow
is due to that the function $1-{\dd \pi}/{\dd \mu}$, which may not be accessible or differentiable.
To address this,
we consider a flow equation with a smooth approximation via the integral operator $\Tkmu$ defined in \eqref{eq:def-int-op}.
Importantly, we do not kernelize the
so-called Wasserstein velocity field
$\nabla (1-\frac{\dd \pi}{\dd \mu})$, but rather the \emph{force field} $1-\frac{\dd \pi}{\dd \mu}$.
We define the following:
\begin{definition}\label{def:force-kernelization}
    [Force-kernelization]
    We refer to the process of replacing
    dual force vector $\xi$,
    of the corresponding Wasserstein gradient flow equation,
    by the kernelized dual force vector $\Tkmu \xi$ as \emph{force-kernelization},
    i.e.,
    \begin{align}\label{eq:force-kernelization-of-wasserstein-gradient-flow}
        \dot \mu = \DIV(\mu \nabla \xi )
        \overset{\text{force-kernelize}}{
            \xrightarrow{\hspace{0.1em}\sim\!\!\sim\!\!\sim\!\!\sim\!\!\sim\!\!\sim\!\!\sim\!\!\sim\!\!\sim\!\!\sim\!\!\sim\!\!\sim\!\!\sim\!\!\sim\!\!\sim\!\!\sim\!\!\sim\!\!\sim\hspace{0.1em}}
        }
        \dot \mu = \DIV(\mu \nabla \Tkmu \xi )
        .
    \end{align}
\end{definition}
This is in contrast to the \emph{velocity-kernelization}commonly used in the literature.
The resulting \emph{force-kernelized flow equation}
of
\eqref{eq:vanilla-wasserstein-rkl-gfe}
is given by
\begin{align}
    \dot \mu = 
    \DIV\left(\mu \nabla \Tkmu \left(1-\frac{\dd \pi}{\dd \mu}\right)\right)
    .
    \label{eq:kernelized-gfe-reverseKL}
\end{align}
Note that a kernelized flow is not necessarily a gradient flow.
However, \eqref{eq:kernelized-gfe-reverseKL} is indeed a gradient flow and
has aleady been used in machine learning, albeit without noting the connection with inclusive KL.%
We observe the following.
\begin{theorem}
    [Flow equation~\eqref{eq:kernelized-gfe-reverseKL} has a Wasserstein gradient structure]
    Suppose that initial condition
    satisfies $\pi \ll \mu$,
    \ie $\pi$ is absolutely continuous with respect to $\mu$.
    Then, 
    \eqref{eq:kernelized-gfe-reverseKL}
    coincides with
    the Wasserstein gradient flow equation of the MMD~\eqref{eq:wgf-mmd-pde},
    \begin{align}
        \dot \mu = 
              \operatorname{div}\left(\mu 
              \int 
              \nabla_2 K(x,\cdot )\dd \left(\mu-\pi\right)(x)
              \right)
              \tag{MMD-WGF}
              \label{eq:wgf-mmd-pde}
    \end{align}
    \label{thm:equivalence-of-gradient-flow-equations}
\end{theorem}
where $\nabla _2$ denotes the differentiation with respect to the second variable.
Intuitively,
\eqref{eq:wgf-mmd-pde} and hence \eqref{eq:kernelized-gfe-reverseKL} have the same gradient structure:
\begin{align}
    \begin{cases}
        \textrm{{Space} :}& 
        \mathcal P 
        \\
        \textrm{Energy functional} :& {F(\cdot):=\frac12 \mmd^2(\cdot , \pi)}
        \\
        \textrm{Dissipation Geometry} :& { \text{Wasserstein}\ \bbKotto}
    \end{cases}
    \label{eq:gradient-structure-mmd-wgf}
\end{align}
As discussed above, the vanilla \eqref{eq:vanilla-wasserstein-rkl-gfe} cannot be directly used to derive algorithms due to the non-smooth nature of the function $1-{\dd \pi}/{\dd \mu}$.
Now, since \eqref{eq:kernelized-gfe-reverseKL}'s flow equation
coincides with \eqref{eq:wgf-mmd-pde},
our theory provides the insight: \emph{minimizing MMD through \eqref{eq:wgf-mmd-pde} is equivalent to simulating a kernelized Wasserstein gradient descent to minimize the inclusive KL}.
Then, we can make use of
numerous implementations that have already been developed for MMD-minimization,
see, e.g., \citep{arbel_maximum_2019,chizatMeanFieldLangevinDynamics2022,futamiBayesianPosteriorApproximation2019,hagemann2023posterior,neumayer2024wasserstein,galashovDeepMMDGradient2024,gladin2024interaction,chenDeregularizedMaximumMean2024,belhadji2025weighted}.
Thus, the dual-force-kernelized gradient flow~\eqref{eq:kernelized-gfe-reverseKL} provides an implementable approximation of \eqref{eq:vanilla-wasserstein-rkl-gfe}.

Summarizing the results so far,
we offer
insights into both score-based (e.g., requiring evaluation of the score function $\nabla \log \pi$) and sample-based (e.g., assuming access to samples from $\pi$) (Bayesian) inference and sampling, providing a unifying Wasserstein gradient flow perspective on these methods in Bayesian computation based on inclusive KL inference~\eqref{eq:inclusiveKL}.
Our insights provide a first-principles interpretation of these methods via gradient flows.
We also note that a wider class of gradient flows can be characterized using the kernel Stein discrepancy; see the appendix.

\section{Fisher-Rao gradient flows of the inclusive KL functional}
While recent machine learning applications primarily focus on the Wasserstein geometry,
we emphasize that the gradient flow theory is more general.
Prominent examples include the Fisher-Rao and Hellinger geometries~\citep{hellingerNeueBegrundungTheorie1909,kakutaniEquivalenceInfiniteProduct1948,rao1945information,bhattacharyyaMeasureDivergenceTwo1946}, which provide a different yet extremely impactful perspective
on statistical inference
and optimization.
They form an important building block for the Wasserstein-Fisher-Rao gradient flow in the next section.

In this subsection, we first analyze
and uncover a few remarkable properties of
the Fisher-Rao gradient flows of the inclusive KL divergence.
Then, we 
establish a precise
connection to the MMD gradient flow of the MMD functional.
This connection was not known prior to this work, yet some existing machine learning algorithms
have already provided empirical implications of such Fisher-Rao gradient flow.

Our starting point is to replace the Wasserstein dissipation geometry
in the gradient structure~\eqref{eq:gradient-structure-rkl-wgf}
with the Fisher-Rao dissipation geometry,
defined using the Fisher-Rao Onsager operator,
$$\mathbb K_\FR(\rho): T^*_\rho \calM \to T_\rho \calM, \xi \mapsto 
\rho\left(\xi - \int \xi \dd \rho\right).$$
The resulting gradient structure is
\begin{align}
    \begin{cases}
        \textrm{{Space} :}& 
        \text{prob. measures } \calP  
        \\
        \textrm{Energy functional} :& {
        \DKL(  \pi| \cdot)}
        \\
        \textrm{Dissipation Geometry} :& { \text{Fisher-Rao } \bbK_\FR}
    \end{cases}
    \label{eq:gradient-structure-revKL-SHe-gf-prob}
\end{align}
Note that the Fisher-Rao space is also
referred to as the spherical Hellinger space
by~\citet{LasMie19GPCA} 
considering the historical development.
Interestingly, under the inclusive KL divergence functional, the Hellinger gradient flow over $\Mplus$ stays within the probability space $\mathcal P$ if initialized therein,
\ie a projection is not needed in our case; see the appendix for more details.
Previously, flows in the Fisher-Rao space have been studied in ML applications under the name of birth-death dynamics; see \citep{luAcceleratingLangevinSampling2019,rotskoff2019neuron,kim2024transformers} for applications and further discussions.

We summarize some important properties of the Fisher-Rao gradient systems in the following.
As is already known due to the McCann condition, the inclusive KL is not generally geodesically convex in the Wasserstein geometry (see \cite{mielke2025hellinger} for details).
However, this notion of convexity has changed since we are now considering the Fisher-Rao geometry.

Here, we exploit a couple of recent analysis results, we now show the geodesic convexity of the inclusive KL divergence
along the Fisher-Rao gradient flow over both general measures.
Specifically,
\citet{carrillo2024fisher} showed that, for a general
$\varphi$-divergence $\rmD_\varphi(\cdot |\pi)$,
a sufficient condition for it to be geodesically $\lambda$-convex ($\lambda > 0$)
is that $\varphi''(s) = s^{p-2}$ for $p\leq 0$.
In fact, \citet{mielke2025hellinger} showed a refined necessary and sufficient condition for gradient dominance (Polyak-\L{}ojasiewicz inequality) to hold
with a positive constant $\lambda > 0$
for
$\varphi(s)=\frac{1}{p(p-1)}\left(s^p-ps+p-1\right)$
is $p\leq \frac12$ (note that $p=0$ recovers the inclusive KL divergence in the limit).
More concretely, from the analysis of \cite{carrillo2024fisher,mielke2025hellinger}, we directly conclude the following regarding the inclusive KL divergence.
\begin{theorem}\label{thm:fkl-geodesic-convexity}
    The inclusive KL divergence $\KL(\pi | \cdot)$ is geodesically $\lambda$-convex in the Fisher-Rao space over general probability measures $\left(\calP(\bbR^d), \FR\right)$ for some $\lambda > 0$.
    Furthermore, it satisfies a global gradient dominance (Polyak-\L{}ojasiewicz) condition for a strictly positive constant in the Fisher-Rao space.
\end{theorem}
We remind the readers again about the naming convention here: the Fisher-Rao space or manifold is used to refer to the space of general probability measures rather than parameter manifolds; see Remark~\ref{remark:distinction-between-FR-and-He}.
Later, we study the restricted gradient flow on the Gaussian manifold in Section~\ref{sec:gaussian-manifold}.

Continuing from Theorem~\ref{thm:fkl-geodesic-convexity}, 
the gradient flow equation can be explicitly solved to obtain:
\begin{proposition}
    [FR gradient flow equation of inclusive KL]
    The gradient structure \eqref{eq:gradient-structure-revKL-SHe-gf-prob} generates the flow equation
    \begin{align}
        \dot \mu 
            =
            \pi - \mu
            .
            \label{eq:gradient-structure-revKL-he-gf-reaction}
        \end{align}
    Its closed-form solution is given by
    \begin{align}
        \mu_t = e^{-t}\mu_0 + (1-e^{-t})\pi
        .
        \label{eq:gradient-structure-revKL-he-gf-sol}
    \end{align}
Furthermore,
    suppose the initial datum is a probability measure $\mu_0\in \calP$.
Then, the solution of the Hellinger flow equation over $\Mplus$ generated by the Hellinger gradient structure of the inclusive KL energy, $\left(\Mplus,\,\DKL(\pi|\cdot),\,\mathrm{Hellinger}\right)$ (see Appendix), coincides with that of the
Fisher-Rao (a.k.a. Spherical Hellinger) gradient structure \eqref{eq:gradient-structure-revKL-SHe-gf-prob}.
\label{prop:revKL-FR-gf-equiv}
\end{proposition}
This result
characterizes an interesting feature of the \emph{inclusive-KL-Fisher-Rao flow:
it traverses along a straight line despite the Riemannian structure of the Fisher-Rao geometry}.

As a consequence of 
the abstract result in Theorem~\ref{thm:fkl-geodesic-convexity},
the following result is immediate.
\begin{theorem}
[Exponential Decay of inclusive-KL divergence]
There exists a constant $c>0$ such that the following Polyak-\Loj functional
inequality holds globally.
\begin{align}
    \biggl\|1 - \frac{\dd \pi}{\dd\mu}\biggr\|^2_{L^2_{\mu}}
    \geq 
     c\cdot   \operatorname{\DKL}(\pi | \mu) 
     ,
     \quad \forall \mu \in \Mplus
     .
    \label{eq:Loj-FR}
\end{align}

Furthermore,
the inclusive KL
satisfies the exponential decay estimate
along the gradient flow
\[
\DKL(\pi | \mu(t)) \leq e^{- t} \DKL(\pi | \mu_0) \text{ for all }t > 0.
\]
\label{thm:exponential-decay-of-inclusive-KL-divergence}
\end{theorem}

We emphasize that
Theorem~\ref{thm:exponential-decay-of-inclusive-KL-divergence}
is global and does not require the assumption of a uniform bound on the density ratio ${\dd \mu_0}/{\dd \pi}$ such as in
\citep{lu2023birth}.
This result
indicates a remarkable feature of the inclusive KL divergence: its Fisher-Rao gradient flow is capable of creating mass from zero-mass regions of $\pi$.
In machine learning and statistics, this is a highly desired feature as we often need to locate the support of the target measure $\pi$.
An intuition of the distinction between the exclusive KL 
and the inclusive KL (Theorem~\ref{thm:exponential-decay-of-inclusive-KL-divergence})
is indeed given by difference of their entropy generator slopes near the zero-mass region.

\paragraph{Connection with kernel methods}
It is also worth mentioning that,
Theorem~\ref{thm:exponential-decay-of-inclusive-KL-divergence}
does not require assumptions such as that the target $\pi$ is log-concave, as in the Wasserstein geometry.
In this sense, the Fisher-Rao gradient flow of the inclusive KL divergence is more generous than that in the Wasserstein geometry.

There is an interesting \emph{coincidence} of
the gradient systems \eqref{eq:gradient-structure-revKL-SHe-gf-prob}
and existing machine learning algorithms.
Let us consider a seemingly unrelated gradient system where both the energy functional $F$ and the dissipation geometry to be MMD, i.e.,
\begin{align}
    \begin{cases}
        \textrm{{Space} :}& 
        \mathcal P 
        \\
        \textrm{Energy functional} :& {F(\cdot):= \mmd^2(\cdot, \pi)}
        \\
        \textrm{Dissipation Geometry} :& { \text{MMD}}
    \end{cases}
    \label{eq:gradient-structure-revKL-MMD-gf-pure}
\end{align}
\begin{proposition}
    \label{prop:MMD-MMD-gf-equiv-to-revKL-HE-gf}
The MMD-MMD gradient flow equation,
generated by the gradient system \eqref{eq:gradient-structure-revKL-MMD-gf-pure},
coincides with the inclusive-KL-Fisher-Rao gradient flow equation~\eqref{eq:gradient-structure-revKL-he-gf-reaction}.
Consequently, MMD decays exponentially along the solution $\mu_t$ of~\eqref{eq:gradient-structure-revKL-he-gf-reaction},
\ie
${\mmd(\mu_t, \nu)\leq e^{- t}\cdot  \mmd(\mu_0, \nu)}$.
\end{proposition}
From a dynamical system perspective,
this result also shows that the MMD is a \emph{Lyapunov functional} for the inclusive-KL Fisher-Rao flow.
We make this point precise mathematically in the following result.
\begin{theorem}
    [MMD as Lyapunov functional for inclusive-KL Fisher-Rao flow]
    Suppose the kernel $K$ is bounded and satisfies the integrally strict positive-definite (ISPD) condition.
    The dissipation quantity, \ie the time derivatives of the inclusive KL and the MMD satisfies
    \begin{align}
        \frac{\dd }{\dd t}
        \left(
            \DKL(\pi | \mu_t)
        \right) 
        \leq 
        c\cdot \frac{\dd }{\dd t} \mmd^2(\mu_t, \pi)
        \text{ for some } c>0
        .
    \end{align}
    Furthermore, the equilibrium satisfies
        $\frac{\dd }{\dd t}
        \left(
            \DKL(\pi | \mu_t)
        \right) 
        =0
        \iff
         \frac{\dd }{\dd t} \mmd^2(\mu_t, \pi)=0$.
\label{thm:lyapunov-MMD}
\end{theorem}
    \begin{proof}
        This proof is by the energy dissipation equality.
    \end{proof}

So far, we have shown that the same flow equation~\eqref{eq:gradient-structure-revKL-he-gf-reaction}
has
two different gradient structures: MMD and Fisher-Rao.
Such instances are well-known in PDE literature.
To implement a practical algorithm
via simulating the gradient flow~\eqref{eq:gradient-structure-revKL-MMD-gf-pure},
\citet{gladin2024interaction} proposed to use the following
minimizing movement scheme:
\begin{align}
    \mu^{\ell+1}
    &
    \gets\argmin_{\mu\in\cal P}
    \frac12 \mmd^2(\mu, \pi ) + \frac1{2\eta}\mmd^2(\mu, \mu^l)
    .
    \tag{MMD-MMD}
    \label{eq:mmd-mmd-JKO}
\end{align}
Using Proposition~\ref{prop:MMD-MMD-gf-equiv-to-revKL-HE-gf},
we can obtain an
interesting insight that connects kernel methods and information geometry.
\begin{proposition}
    Suppose the kernel $K$ is bounded and ISPD.
    Then, $\mu^*$ is a solution of the variational problem \eqref{eq:mmd-mmd-JKO}
    if and only if
    it is a solution of
\begin{align}
    \argmin_{\mu\in\cal P} \DKL(\pi | \mu) + \frac1{\eta}\DKL(\mu^l | \mu)
    .
    \label{eq:revKL-FR-JKO-via-MMD}
\end{align}
    \label{prop:revKL-FR-JKO-via-MMD}
\end{proposition}

In addition to \eqref{eq:revKL-FR-JKO-via-MMD},
we can generalize the variational problem to general $\varphi$-divergence:
    $\displaystyle\argmin_{\mu\in\cal P} \DKL(\pi | \mu) + \frac1{\eta}\mathrm{D}_{\varphi}(\mu^l | \mu)$,
which includes the special case when $\mathrm{D}_{\varphi}$ is the squared Hellinger distance; see also Remark~\ref{remark:distinction-between-FR-and-He}.

\eqref{eq:mmd-mmd-JKO} 
can also be viewed as
a scaled instance of a MMD Barycenter problem, whose solution can be approximated using existing algorithms proposed by \citet{cohenEstimatingBarycentersMeasures2021,gladin2024interaction}.
From this paper's perspective,
Proposition~\ref{prop:revKL-FR-JKO-via-MMD} shows that their numerical algorithms 
actually also solves the variational problem~\eqref{eq:revKL-FR-JKO-via-MMD}.
Therefore, they can also be used to simulate a gradient flow that minimizes the inclusive-KL functional on the Fisher-Rao geometry.

\begin{remark}[Naming convention: Fisher-Rao vs. Hellinger]
    \label{remark:distinction-between-FR-and-He}
        Strictly speaking, we make the following distinction regarding the naming convention
        of the Fisher-Rao and Hellinger metrics.
        For more details, see the discussion in \cite{mielke2025hellinger,mielke2025some}.
        \begin{itemize}
            \item 
            Fisher-Rao (FR): a distance between parameters of the (exponential-family) distributions.
            \item
            Hellinger (He): a 
            $\varphi$-divergence/distance over positive measures.
            \item
            spherical Hellinger (SHe): a distance induced by restricting the Hellinger geodesics to the probability measures; also called Bhattacharya distance by \citet{rao1945information} after its first introduction
            by \citet{bhattacharyyaMeasureDivergenceTwo1946}.
             We can recover the equivalence between SHe and FR if we consider the
            trivial parameterization of the probability measure by itself (infinite-dimensional).
        \end{itemize}
\end{remark}

\section{Unbalanced transport: Wasserstein-Fisher-Rao gradient flows}
\label{sec:wfr-gfe}

The Wasserstein geometry endows us with the mechanism to transport mass.
On the other hand,
the Fisher-Rao geometry lets us create and destroy mass.
One major development in optimal transport theory is 
the combination of both via unbalanced transport, invented independently by~\citet{chizatInterpolatingDistanceOptimal2018,chizat_unbalanced_2019,liero_optimal_2018,kondratyevNewOptimalTransport2016}.
    The resulting metric between two non-negative measures is known as the Wasserstein-Fisher-Rao (WFR) distance, also known as the Hellinger-Kantorovich distance,
    defined via the entropic transport problem~\citep{liero_optimal_2018}
    \begin{align*}
        \WFR^2(\mu_1, \mu_2)
        =
          \min _{\Pi\in \Gamma(\mu_1, \mu_2)}
          \Biggl\{
          \alpha \int  c \dd {\Pi}
          +
          \beta { \DKL}(\pi_1|\mu_1)
          +
          \beta{ \DKL}(\pi_2|\mu_2)
          \Biggr\}
    \end{align*}
    where $\alpha, \beta>0$ are two scaling parameters.
    $\Gamma(\mu_1, \mu_2)$ is the set of all positive measures with marginals $\mu_1$ and $\mu_2$.
    $c$ is the transport cost in the standard Wasserstein distance and $\mathrm{D}_\varphi$ is the $\varphi$-divergence (defined in~\eqref{eq:general-divergence}).
In this paper, we define the WFR gradient structure via the Onsager operator:
the WFR Riemannian metric tensor is an inf-convolution of the Wasserstein tensor and the Fisher-Rao tensor
$\mathbb G_{\WFR}(\mu) = \mathbb G_{W}(\mu) \square \mathbb G_{\FR}(\mu)$
\citep{chizat_unbalanced_2019,liero_optimal_2018,gallouet2017jko}.
By the Legendre transform, its inverse, the Onsager operator, is given by the sum
$\mathbb K_{\WFR}(\mu) = \mathbb K_{W}(\mu) + \mathbb K_{\FR}(\mu)$.
For conciseness, we only focus on the case of
WFR distance restricted to the space of probability measures by default.
Therefore,
the WFR distance should technically be referred to as the spherical Hellinger-Kantorovich distance.
Let us now consider the following gradient structure in the WFR space
\begin{align}
    \begin{cases}
        \textrm{{Space} :}& 
        \mathcal P
        \\
        \textrm{Energy functional} :& {\text{inclusive KL: } \DKL(\pi | \cdot)}
        \\
        \textrm{Dissipation Geometry} :& { \text{(spherical) WFR } \bbK_{\WFR}}
    \end{cases}
    \label{eq:gradient-structure-revKL-WFR-pure}
\end{align}
Using \eqref{eq:wfr-gfe-revKL} and the results in the previous two sections,
the WFR gradient flow equation generated by \eqref{eq:gradient-structure-revKL-WFR-pure} is given by the reaction-diffusion-type PDE
\begin{align}
    \dot \mu =  \underbrace{\alpha \DIV \left(\mu\nabla 
    \left(
        1 - \frac{\dd \pi}{\dd \mu}
    \right)
    \right)}_{\text{Wasserstein: transport}} - \underbrace{\beta \mu \cdot
    \left(
        1 - \frac{\dd \pi}{\dd \mu}
    \right)}_{\text{Fisher-Rao: birth-death}}
    .
    \label{eq:wfr-gfe-revKL}
\end{align}
The derivation is standard; cf. the aforementioned references.
Exploiting the unique properties established in Theorem~\ref{thm:exponential-decay-of-inclusive-KL-divergence},
we can conclude the following.
\begin{corollary}
    The inclusive KL divergence functional decays
    exponentially towards zero along the solution of the PDE~\eqref{eq:wfr-gfe-revKL}.
    \label{cor:wfr-gfe-revKL-decay}
\end{corollary}
While this result renders the WFR gradient flow equation an attractive candidate for algorithm design,
we again cannot simulate~\eqref{eq:wfr-gfe-revKL} due to the function $1 - \dd\pi/\dd\mu$.
To address this,
we now follow \eqref{eq:kernelized-gfe-reverseKL} to kernelize the generalized force in the transport velocity
\begin{align}
    \dot \mu =  
    \alpha\cdot \operatorname{div}\left(\mu \int \nabla_2 K(x,\cdot )\dd \left(\mu-\pi\right)(x)\right)
    - \beta \cdot
    \left(
        \mu -{ \pi}
    \right)
    .
    \tag{IFT-GF}
    \label{eq:wfr-gfe-revKL-kernelized}
\end{align}
Due to
Proposition~\ref{prop:MMD-MMD-gf-equiv-to-revKL-HE-gf},
we immediately find that:
\begin{corollary}
    \eqref{eq:wfr-gfe-revKL-kernelized}
    is the gradient flow equation of the squared MMD functional,
    \ie with the gradient structure
    \begin{align}
        \begin{cases}
            \textrm{{Space} :}& 
            \mathcal P 
            \\
            \textrm{Energy functional} :& {\frac12 \mmd^2(\cdot, \pi)}
            \\
            \textrm{Dissipation Geometry} :& { \text{Interaction-force transport (IFT)~\citep{gladin2024interaction}
            }}
        \end{cases}
        \label{eq:gradient-structure-revKL-IFT-pure}
    \end{align}
    \label{cor:revKL-MMD-wfr-equiv-g-structure}
\end{corollary}
This has recently been studied
under the name of
interaction-force transport (IFT) gradient flow by \citet{gladin2024interaction}.
It has been shown to practically accelerate and improve the performance of the MMD minimization task
with proven guarantees.
Here, we have shown that \emph{the IFT gradient flow in \citet{gladin2024interaction} is
an approximation to the Wasserstein-Fisher-Rao gradient flow of the inclusive-KL functional}.
\citet{gladin2024interaction} have shown that 
MMD decays exponentially along the
    solution $\mu_t$ to the PDE~\eqref{eq:wfr-gfe-revKL-kernelized}, \ie
            $\mmd(\mu_t, \nu)
        \leq 
        e^{- \beta t}\cdot  \mmd(\mu_0, \nu)$
        .
An important aspect is that
this convergence does not rely on the so-called log-concavity of the target distribution $\pi$; cf. \citep{chewiStatisticalOptimalTransport2024}.

\section{Gaussian gradient flows}
\label{sec:gaussian-manifold}
The discussion in this paper focuses on the space of general probability measures.
In variational inference \cite{jordanIntroductionVariationalMethods1999,
wainwrightGraphicalModelsExponential2008, blei_variational_2017},
practitioners are often interested in the Gaussian subspace $\calN^d$, which is a subspace of the general probability measure space.
An extremely effective algorithm, termed natural gradient descent \citep{amari1998natural,amariMethodsInformationGeometry2000,khanBayesianLearningRule2023,shen2024variational},
is the stable of the GVI framework.
The continuous-time dynamics of GVI optimization can be described a system of ODEs that evolve the parameter values $(m_t, \Sigma_t)$ that parameterize a Gaussian measure, sometimes known as the natural gradient flow.
We extend our analysis to the Gaussian space to provide a principled mathematical analysis foundation for Gaussian VI with the inclusive KL divergence, which was widely studied in machine learning literature, e.g. \citep{naessethMarkovianScoreClimbing2020,jerfelVariationalRefinementImportance2021,mcnamaraSequentialMonteCarlo2024,zhangTransportScoreClimbing2022}.

In this section, we consider the gradient structure:
\begin{align}
    \begin{cases}
        \textrm{{Space} :}& 
        d\textrm{-dimensional Gaussian measures } \mathcal{N}^d(\bbR^d)
        \\
        \textrm{Energy functional} :& {\text{inclusive KL: } \DKL(\pi | N(m, \Sigma))}
        \\
        \textrm{Dissipation Geometry} :& { \text{Fisher--Rao} }
    \end{cases}
\end{align}
The ODE flow equations of the Fisher-Rao Gaussian gradient flow with the \emph{exclusive} KL driving energy were analyzed in detail in \citep{chenGradientFlowsSampling2023,lieroEvolutionGaussiansHellingerKantorovichBoltzmann2025}.
We now study the Fisher-Rao Gaussian gradient flow with the \emph{inclusive} KL driving energy.
First, using the gradient structure provided by \citet{lieroEvolutionGaussiansHellingerKantorovichBoltzmann2025}, we can straightforwardly derive Fisher-Rao and
Bures-Wasserstein Gaussian \citep{lambertVariationalInferenceWasserstein2022,takatsu2011wasserstein} gradient flow are:
Suppose that the target distribution $\pi$ is not necessarily Gaussian, we denote its first and second moments as:
\begin{align}
    m_\pi := \mathbb{E}_\pi[x],
    \quad
    \Sigma_\pi := \mathbb{E}_\pi[(x - m_\pi)(x - m_\pi)^T].
    \label{eq:gvi-gfe-notation}
\end{align}
Using this notation, the inclusive KL divergence between a Gaussian measure $N(m_t, \Sigma_t)$ and the target distribution $\pi$ is given by:
\begin{align}\label{eq:KL-closed-form-Gaussian}
    \DKL(\pi | N(m_t, \Sigma_t))
    = \frac12 \operatorname{tr} \left( \Sigma_t^{-1} \Sigma_\pi \right) 
    + \frac12 \left( m_\pi - m_t \right)^T \Sigma_t^{-1} \left( m_\pi - m_t \right)
    + \frac12 \log \det \Sigma_t
    + \const.
\end{align}

A direct application of the gradient flow structure derived in \citet{lieroEvolutionGaussiansHellingerKantorovichBoltzmann2025} gives the following proposition:

\begin{proposition}
    \label{prop:gvi-gfe}
    The Fisher-Rao Gaussian gradient flow equation of the inclusive KL divergence is given by the following system of ODEs:
    \begin{align}
        \begin{split}
        \dot m_t &
        =
        -  (m_t - m_\pi),
        \\
        \dot \Sigma_t &
        =
        -\left(   \Sigma_t - \Sigma_\pi - (m_\pi - m_t)(m_\pi - m_t)^T \right)
        .
        \end{split}
        \label{eq:gvi-gfe}
    \end{align}
    The Bures-Wasserstein Gaussian gradient flow equations of the inclusive KL divergence are given by the following system of ODEs:
\begin{align}
    \begin{split}
        \dot m_t &= - \Sigma_t^{-1}(m_t - m _\pi ) 
        ,
        \\
        \dot \Sigma_t &= \Sigma_t^{-1}\left(  
            \Sigma_\pi - \Sigma_t
            + 
            \left(m_\pi - m_t\right)\left(m_\pi - m_t\right)^T
        \right)
        +
        \left(  
            \Sigma_\pi - \Sigma_t
            + 
            \left(m_\pi - m_t\right)\left(m_\pi - m_t\right)^T
        \right)
        \Sigma_t^{-1}
        .
    \end{split}
    \label{eq:gvi-bures-wasserstein-gfe-2}
\end{align}
Consequently, the Wasserstein-Fisher-Rao Gaussian gradient flow equations is obtained by combining the Fisher-Rao and Bures-Wasserstein gradient flow equations. We omit the formula to avoid redundancy; see \citep{lieroEvolutionGaussiansHellingerKantorovichBoltzmann2025} for details.
The covariance equation in \eqref{eq:gvi-bures-wasserstein-gfe-2} is, just like in the reverse KL setting \citep{lambertVariationalInferenceWasserstein2022},
a Lyapunov equation of a linear dynamical system.
Compared with the systems of ODEs in the Bures-Wasserstein~\cite{lambertVariationalInferenceWasserstein2022}, Fisher-Rao~\cite{chenGradientFlowsSampling2023}, and Wasserstein-Fisher-Rao~\cite{lieroEvolutionGaussiansHellingerKantorovichBoltzmann2025}, the ODE systems above involve the target moments $m_\pi, \Sigma_\pi$ directly.

\end{proposition}
The derivation of \eqref{eq:gvi-gfe} is given in Appendix.
The intuition of the above formula is clear: The ODEs drive the Gaussian state $N(m_t, \Sigma_t)$ towards the moment-matching target.

To illustrate the evolution of the Fisher-Rao Gaussian ODE using its explicit analytical solution derived in Proposition~\ref{cor:fr-ode-solution}, we consider three different target distributions $\pi$. The single-Gaussian case is already illustrated in the introduction (Figure~\ref{fig:fr-ode-intro}); here we focus on the two more challenging targets: a Gaussian Mixture Model (GMM) with 3 components (Figure~\ref{fig:fisher-rao-explicit-gmm}, left) and a joker/double banana target (Figure~\ref{fig:fisher-rao-explicit-banana}, right).

\begin{figure}[ht]
    \centering
    \begin{minipage}[t]{0.48\columnwidth}
        \vspace{0pt}
        \centering
        \includegraphics[width=\linewidth, trim={0 0 250 0}, clip]{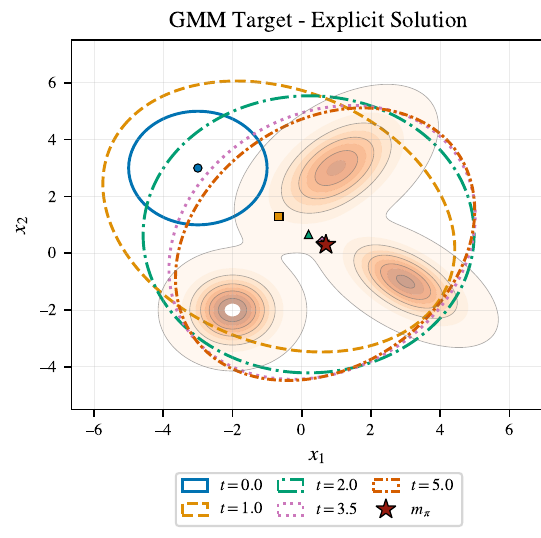}\\
        {\footnotesize(a) GMM target}
    \end{minipage}%
    \hfill%
    \begin{minipage}[t]{0.48\columnwidth}
        \vspace{0pt}
        \centering
        \includegraphics[width=\linewidth, trim={0 0 250 0}, clip]{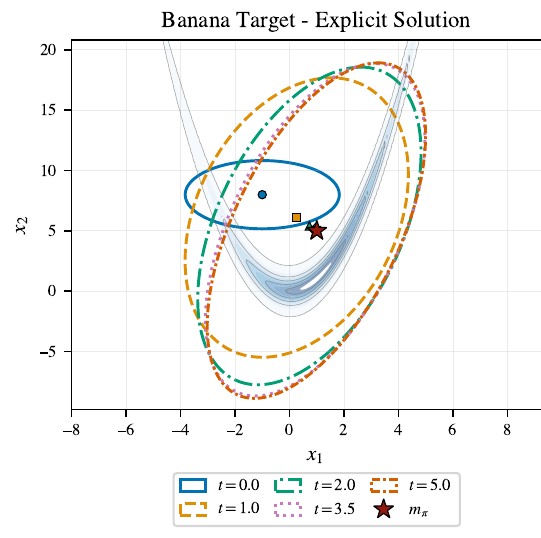}\\
        {\footnotesize(b) banana target}
    \end{minipage}
    \caption{Explicit Fisher--Rao ODE solution evolution for two targets: \emph{(a)} a $3$-component Gaussian mixture and \emph{(b)} a joker/double-banana target. In each panel the target density $\pi$ is shown as a background contour with the star marking $m_\pi$, and the flow Gaussian contours are drawn for $t=0,1.0,2.0,3.5,5.0$. In both cases the inclusive KL functional (ignoring the additive constant), computed in closed form via~\eqref{eq:KL-closed-form-Gaussian}, decays rapidly and monotonically, essentially leveling off by $t\approx 2$; the decay curves are omitted to save space.}
    \label{fig:fisher-rao-explicit-gmm}%
    \label{fig:fisher-rao-explicit-banana}
\end{figure}

Previous works on Gaussian gradient flows of the exclusive (forward) KL divergence rely on the strong log-concavity of the target distribution $\pi$ to establish the analysis results; see \citep{lambertVariationalInferenceWasserstein2022,chenGradientFlowsSampling2023,lieroEvolutionGaussiansHellingerKantorovichBoltzmann2025}.
To showcase a different perspective under the inclusive KL functional, we now establish the global convergence of the Fisher-Rao-Gaussian gradient flow of the inclusive KL divergence.
\begin{proposition}\label{cor:fr-ode-solution}
    The Fisher-Rao-Gaussian ODE \eqref{eq:gvi-gfe} has unique
    exponentially converging
    solution
    \begin{align}
        m_t & = (1 - e^{-t}) m_\pi + e^{-t} m_0
        ,
        \label{eq:fr-ode-solution-mean}\\
        \Sigma_t
        & =
        e^{-t} \Sigma_0 +
        (1 - e^{-t}) \Sigma_\pi + 
        e^{-t}(1 - e^{-t}) (m_0 - m_\pi)(m_0 - m_\pi)^T
        \label{eq:fr-ode-solution-cov}
        .
    \end{align}
\end{proposition}
We emphasize that it does not require any strong log-concavity assumption as in the Bures-Wasserstein flow of the exclusive KL in \citep{lambertVariationalInferenceWasserstein2022}.
We obtain the insights from the ODE solution:
The mean solution~\eqref{eq:fr-ode-solution-mean} is a straightforward linear interpolation between the initial target means.
The covariance solution~\eqref{eq:fr-ode-solution-cov}, however, is a ``bridge'' between the initial and target covariance, with a correction term $e^{-t}(1 - e^{-t}) (m_0 - m_\pi)(m_0 - m_\pi)^T$ that vanishes at both ends $t=0$ and $t=\infty$.

Thus, this paper provides an additional perspective based on the analysis of gradient flows: the natural gradient flow (Fisher-Rao-Gaussian) of the inclusive KL divergence
admits the flow equation~\eqref{eq:gvi-gfe} and is globally exponentially convergent.
When $\pi$ is non-Gaussian, the flow converges to the moment-matching Gaussian $\hat\pi:=N(m_\pi,\Sigma_\pi)$, the Gaussian minimizer of $\DKL(\pi|\cdot)$.
This further demonstrates a rather different (e.g. no log-concave assumptions) and clean perspective of the gradient flow convergence from the typical exclusive KL settings in the literature.
Due to space constraints, we leave the practical implementation aspect of variational inference based on our theory to future work.

\section{Local-estimator gradient flows}
\label{sec:local_estimate}
The notion of ``kernelizing'' gradient flow is often summoned to describe using integral operator to approximate Wasserstein gradient flow velocity.
In contrast, using our \emph{force-kernelization} principle~\eqref{eq:force-kernelization-of-wasserstein-gradient-flow},
we now present a
but different construction using
the following local regression estimator
(a.k.a. Nadaraya-Watson)
of the variational derivative:
\begin{align*}
     \hat \theta_0(x)=
    \argmin_{\theta \in \bbR}
    \biggl\{
            \int \mu(x') K(x'-x) 
            \biggl|\theta  -   \dFdmut(x')\biggr|^2
            \dd x'
        \biggr\}
        ,
\end{align*}
where the energy functional the inclusive KL divergence $F=\DKL(\pi|\cdot)$ as in~\eqref{eq:vanilla-wasserstein-rkl-gfe}.
We again emphasize that we do not fit $\nabla \dFdmut$, but only the force $\dFdmut$ itself.
Using standard nonparametric regression results~\citep{tsybakov_introduction_2009,spokoiny2016nonparametric,zhu2024kernel},
we obtain the resulting flow equation
\begin{align*}
    \dot \mu =
    \DIV\left(
        \mu
    \cdot 
    \nabla 
    \int \frac{K(x'-x)}{\int \mu(x'') K(x''-x) \dd x''} \dd  \biggl( \mu(x') - \pi(x') \biggr) 
     \right)
    .
\end{align*}
However, in the paper, we will not use this flow equation:
Instead of taking gradient of kernels as in \cite{arbel_maximum_2019,korbaKernelSteinDiscrepancy2021},
we now propose
a direct estimator of the Wasserstein velocity
by incorporating a classical property of nonparametric regression
into our \emph{force-kernelization} principle~\eqref{eq:force-kernelization-of-wasserstein-gradient-flow}.
\begin{definition}
    The local-estimator Wasserstein gradient flow (\WGFloc) equation
    of the functional $F$
    is given by
    \begin{align}
        \label{eq:local-estimator-wgf}
        \dot \mu =
        \DIV\left(
            \mu \cdot \hat\theta_1(x)
        \right)
    \end{align}
    where $\hat\theta_1(x)$ is the solution to the nonparametric regression
    \begin{align}
        \bigl(\hat\theta_0,\hat\theta_1\bigr)(x)
        =
        \argmin_{\theta_0\in\bbR,\,\theta_1\in\bbR^d}
        \sum_{i=1}^N
        \left(
            \dFdmut(x_i) - \theta_0 - \theta_1^\top(x_i - x)
        \right)^2
        \,K_h(x_i - x)
        ,
        \label{eq:local-lin-wls-velocity}
    \end{align}
\end{definition}

Remarkably, the local-estimator Wasserstein gradient flow (\WGFloc) admits a nice sample-based estimator
without needing, in particular, the density ratio $\pi/\mu$ or gradient of the kernel:

Let $\{x_i\}_{i=1}^N\sim\mu_t$ be the particles and $\{z_j\}_{j=1}^M\sim\pi$ be
samples from the target, and recall the localized weights notation from Spokoiny~\citep{spokoiny2016nonparametric},
$\omega_i(x)=K_h(x_i-x)\big/\sum_{l=1}^N K_h(x_l-x)$.
For notational convenience, we define the following four sample-based moments:
\begin{gather}
    \begin{aligned}
    F_{0,1}(x)
    &=
    \sum_{i=1}^N \omega_i(x)\,(x_i-x)
    ,
    &
    F_{1,1}(x)
    &=
    \sum_{i=1}^N \omega_i(x)\,(x_i-x)(x_i-x)^\top
    ,
    \end{aligned}\label{eq:spokoiny-design-moments}
    \\
    \begin{aligned}
    S_0(x)
    &=
    1-\frac{N}{M}\sum_{j=1}^M \frac{K_h(z_j-x)}{\sum_{l=1}^N K_h(x_l-x)}
    ,
    \\
    S_1(x)
    &=
    \sum_{i=1}^N \omega_i(x)\,(x_i-x)
    -\frac{N}{M}\sum_{j=1}^M
    \frac{K_h(z_j-x)}{\sum_{l=1}^N K_h(x_l-x)}\,(z_j-x)
    .
    \end{aligned}\label{eq:spokoiny-response-moments}
\end{gather}
\begin{proposition}[Local-estimator WGF velocity in normalized-moment form]
\label{prop:local-estimator-wgf-spokoiny}
    The \WGFloc{} equation of the inclusive KL
    divergence $F=\DKL(\pi|\cdot)$ is
    \begin{align}
        \dot\mu
        =
        \DIV\!\left(\mu\cdot\hat\theta_1(x)\right)
        \text{ for  }
        \hat\theta_1(x)
        =
        \bigl(F_{1,1}(x)-F_{0,1}(x)\,F_{0,1}(x)^\top\bigr)^{-1}
        \bigl(S_1(x)-F_{0,1}(x)\,S_0(x)\bigr),
        \label{eq:local-estimator-wgf-spokoiny}
    \end{align}
    where $\hat\theta_1(x)$ is the local-linear slope of the
    nonparametric estimator~\eqref{eq:local-lin-wls-velocity}
    and admits a sample-based estimator as given in \eqref{eq:spokoiny-design-moments} and \eqref{eq:spokoiny-response-moments}.
\end{proposition}
Here $S_0(x)$ is precisely the Nadaraya--Watson force (a difference of two
    kernel density estimators), and the first sum in $S_1(x)$ equals $F_{0,1}(x)$.

Beyond the local linear estimator, we can also lift \eqref{eq:local-lin-wls-velocity} to a higher-order local \emph{polynomial} estimator
\begin{align}
    \hat\theta(x)
    =
    \argmin_{\theta\in\bbR^{P}}
    \sum_{i=1}^N
    \bigl(
        Y_i - \boldsymbol{\theta}^\top\phi_p(x_i - x)
    \bigr)^2
    \,K_h(x_i - x)
    ,
    \label{eq:local-poly-wls}
\end{align}
where $\phi_p(u)\in\bbR^{P}$ collects the monomials up to degree $p$, and $\boldsymbol{\theta}\in\bbR^{P}$ is the coefficient vector.
Another straightforward alternative is to replace the polynomial with a neural-network velocity model; we leave a this direction to a future detailed study.

\paragraph{Numerical simulations}
We now
showcase that
the combination of the proposed force-kernelization and local regression
is a powerful tool for approximating the inclusive KL -- Wasserstein gradient flow.
In Figure~\ref{fig:local-mmd-comparison}, we compare the local-estimator gradient flow against the vanilla ``MMD-flow'' of \citet{arbel_maximum_2019}.
As can be visually observed from the top block of Figure~\ref{fig:local-mmd-comparison}, the \WGFloc converges markedly faster, with final $\mmd^2$ values sometimes more than an order of magnitude lower than the baseline; in particular, on the banana target it keeps the particles on the curved arms, whereas the MMD-WGF collapses them onto a flat horizontal blob.
We also observe more stable behaviors for our proposed method, such as fewer straying particles
for the same hyperparameter settings.
For further numerical experiments, please refer to Appendix.
\begin{figure}[p]
    \centering
    \includegraphics[width=0.72\textwidth]{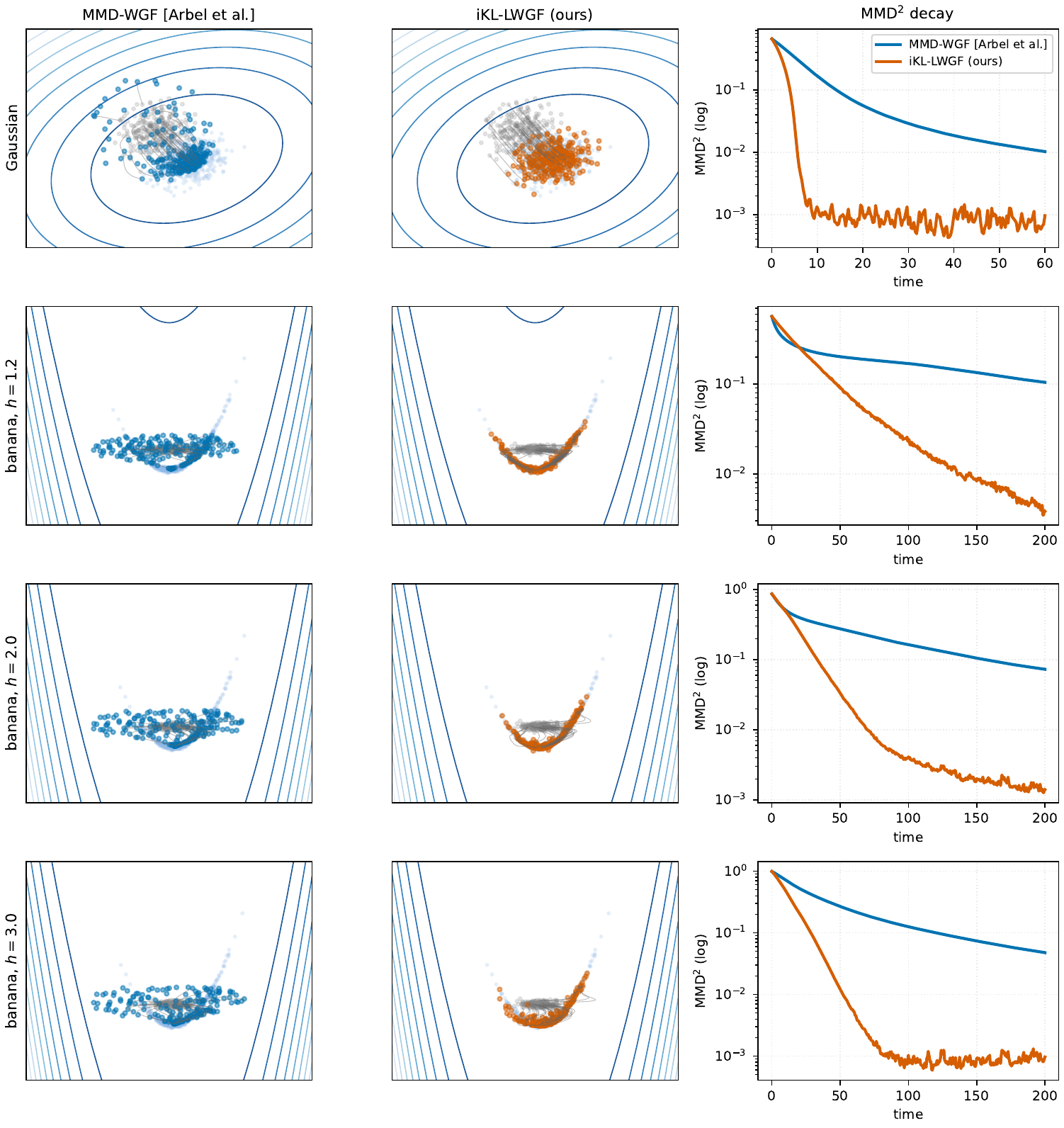}\\[6pt]
    \includegraphics[width=\textwidth]{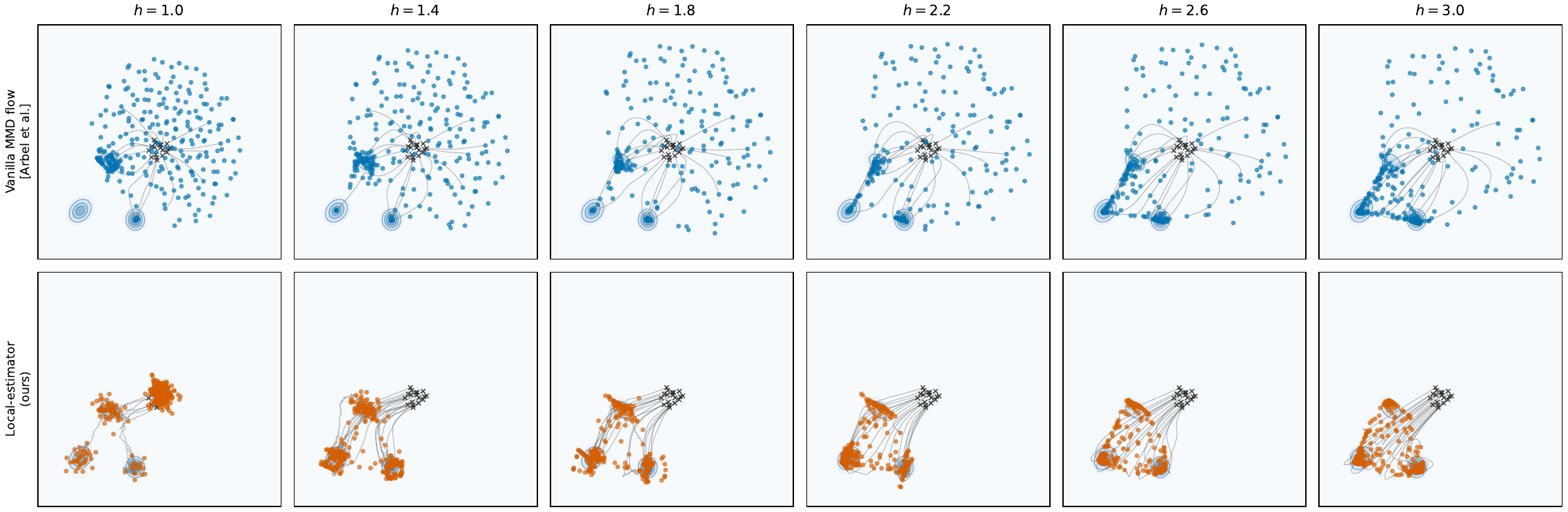}
    \caption{
    \emph{Top block:} a Gaussian and a banana target at three kernel bandwidths $h\in\{1.2,2.0,3.0\}$ (rows); the left two panels of each row show the initial (grey) and final (colored) particles (MMD-WGF and $\mathrm{iKL}\text{-}\WGFloc$, respectively) and the right panel the $\mmd^2$ decay.
    \emph{Bottom block:} a Gaussian mixture across six bandwidths uniformly spaced in $[1.0,3.0]$ (columns), MMD-WGF (top row) and $\mathrm{iKL}\text{-}\WGFloc$ (bottom row), initial location marked by ($\times$).
    The figures are best viewed in the electronic version.
    }
    \label{fig:local-mmd-comparison}\label{fig:local-mmd-gallery-gmm}
\end{figure}
The bandwidth behavior of the \WGFloc{} is illustrated in the bottom block of Figure~\ref{fig:local-mmd-comparison}: for Gaussian-kernel bandwidths uniformly spaced in $[1.0,3.0]$, and from an initialization placed far to the north-east of the three widely-separated modes, the \WGFloc{} (bottom row) concentrates the cloud near the three modes for the moderate-to-large bandwidths in the display, while the smallest bandwidth under-transports the particles under the same fixed step budget.
In contrast, the MMD-WGF~\cite{arbel_maximum_2019} (top row) scatters particles broadly across low-density regions for the same settings.

\section{Further related works and discussions}
Related approximation methods appear in the interacting particle systems literature, such as the blob method~\citep{carrillo2019blob,craig_blob_2023}, and in particle-based gradient descent methods~\citep{daiProvableBayesianInference2016,chizatMeanFieldLangevinDynamics2022}, albeit not concerned with the inclusive KL divergence.
\citet{trillosBayesianUpdateVariational2018} provide a variational perspective for Bayesian update via gradient flows of the KL, $\chi^2$, and Dirichlet energy functionals.
\citet{mauraisSamplingUnitTime2024} seek a velocity field to match the Fisher-Rao flow in the importance sampling setting where the density ratio is accessible.
\citet{vargasTransportMeetsVariational2024} proposed a framework governing many variational Bayesian methods, one part of which is indeed an inclusive KL inference problem.
Several works employ unbalanced transport and its variants for sampling and inference~\citep{luAcceleratingLangevinSampling2019,mroueh_unbalanced_2020,lu2023birth,yanLearningGaussianMixtures2023,gladin2024interaction}.
Instead of approximation via integral operator, a ridge-regression type of gradient flow approximation can also be considered; cf.\ \citep{he2022regularized,zhu2024kernel,nuskenSteinTransportBayesian2024}.
The idea of using local estimator (local regression) for approximating Wasserstein gradient flows was first discussed by this author 
in \cite{zhu2024kernel}
in a different context.
Further implications of our theoretical framework -- including connections to generative modeling (MMD-GANs), the $\chi^2$-divergence WGF, KSD descent as inclusive KL minimization, discrete-time algorithms, and mirror descent -- are presented in Appendix.
One might conjecture that,
standard nonparametric regression results~\citep{tsybakov_introduction_2009,spokoiny2016nonparametric}
can be applied to derive approximation results.
However, proving rigorous $\Gamma$-convergence of gradient flows is 
beyond standard statistical bounds and
mathematically non-trivial. We therefore leave the it to future research.
We refer interested readers to relevant works on $\Gamma$-convergence of gradient flows, e.g., \citep{craig_blob_2023,carrillo2019blob,lu2023birth,zhu2024kernel}.
Recently,
kernelized gradient flows such as SVGD and MMD-WGF have been studied in a more mathematically rigorous way by a few works such as
\citep{chizatQuantitativeConvergenceWasserstein2026,chizatQuantitativeLocalConvergence2026,carrilloSteinVariationalGradient2026}.
We emphasize that the \WGFloc,
which is first proposed in this work,
has already shown great promise and empirical performance in comparison.
A future direction is to study its fine mathematical properties.

\appendix

\section{Additional derivations and proofs}
\label{sec:additional-derivations-and-proofs}

\begin{proof}
    [Proof of Theorem~\ref{thm:equivalence-of-gradient-flow-equations}]
    The verification is a straightforward identification.
    From the right-hand side of \eqref{eq:kernelized-gfe-reverseKL}, we have
    \begin{align*}
        \DIV\left(\mu \nabla \Tkmu \left(1-\frac{\dd \pi}{\dd \mu}\right)\right)
        &=
        \DIV\left(\mu \nabla \left(\int K(x,x')\mu(x') \left(1-\frac{\dd \pi}{\dd \mu}(x')\right)\dd x'\right)\right)
        \\
        &=
        \DIV\left(\mu \nabla \left(\int K(x,x') \left(\mu(x')-\pi(x')\right)\dd x'\right)\right)
        ,
    \end{align*}
    which coincides with the right-hand side of \eqref{eq:wgf-mmd-pde}.
\end{proof}

\begin{proof}
    [Proof of Proposition~\ref{prop:revKL-FR-gf-equiv}]
    The calculation of the flow equation is straightforward via Otto's formalism.
    \begin{align}
        \dot \mu 
        = - \bbK_\FR(\mu) \left(1 - \frac{\dd \pi}{\dd \mu} - Z\right)
    \end{align}
    where $Z$ is the normalization constant.
    Then,
    \begin{align}
        \dot \mu 
        = - \mu  \left(1 - \frac{\dd \pi}{\dd \mu} - Z\right)
        = - (\mu - \pi) 
    \end{align}
    where $Z$ disappears due to that the gradient flow is already mass-preserving.
    Therefore, the flow equation is indeed \eqref{eq:gradient-structure-revKL-he-gf-reaction}.
The ODE solution is obvious.
\end{proof}
\begin{proof}
      [Proof of Theorem~\ref{thm:exponential-decay-of-inclusive-KL-divergence}]
        This is a corollary of the more general result by \citet{mielke2025hellinger}.
        There, they proved that the PL inequality holds for
        a large class of relative entropy functionals including the squared Hellinger distance, the inclusive KL divergence, and the reverse $\chi^2$ divergence.
  Therefore, \eqref{eq:Loj-FR} holds globally.

  Consequently,
        calculating the time-derivative of the
        inclusive KL divergence, we obtain
        \begin{align}
            \frac{\dd }{\dd t} \DKL(\pi | \mu) 
            =
            \langle  1 - \frac{\dd \pi}{\dd \mu}, \dot \mu \rangle
            =
             - \langle  1 - \frac{\dd \pi}{\dd \mu}, \mu\cdot (1 - \frac{\dd \pi}{\dd \mu})\rangle
             \overset{\eqref{eq:Loj-FR}}{\leq} - c\cdot \DKL(\pi | \mu)
            .
        \end{align}
        By Gr\"onwall's Lemma, we obtain the desired estimate.
    \end{proof}

\begin{proof}
    [Proof of Proposition~\ref{prop:MMD-MMD-gf-equiv-to-revKL-HE-gf}]
First, the equivalence between the flow equations is by direct identification -- the flow equations coincide. This is a
consequence of Theorem 3.4 of \citep{gladin2024interaction}.
Then, using this equivalence, the MMD-decay statement follows from Theorem 3.5 of \citep{gladin2024interaction} and their equation (12).
\end{proof}

\begin{proof}
        [Proof of Proposition~\ref{prop:revKL-FR-JKO-via-MMD}]
        We calculate the optimality condition of the following optimization problem~\eqref{eq:revKL-FR-JKO-via-MMD}.
        \begin{align}
            1 - \frac{\dd \pi}{\dd \mu} + \frac1{\eta}\left( 1 -  \frac{\dd \mu^l}{\dd \mu} \right) = 0
            .
        \end{align}
        By the ISPD condition of the kernel $K$, the integral operator $\Tkmu$ is strictly positive-definite.
        Therefore,
        let $\Tkmu$ act on the both sides of the equation above, we have
        \begin{align}
            \Tkmu \left( 1 - \frac{\dd \pi}{\dd \mu} \right) + \frac1{\eta}\Tkmu \left( 1 -  \frac{\dd \mu^l}{\dd \mu} \right) = 0
            ,
        \end{align}
        which coincides with the optimality condition of the variational problem \eqref{eq:mmd-mmd-JKO}
        given Radon-Nikodym derivatives exist.
\end{proof}

\begin{proof}
    [Proof of Corollary~\ref{cor:wfr-gfe-revKL-decay}]
    The proof is by exploiting the inf-convolution structure of the WFR flow.
    By taking the time-derivative of the inclusive KL divergence, we have
    \begin{multline*}
        \frac{\dd }{\dd t} \DKL(\pi | \mu) 
        =
        \langle 1- {\dd \pi}/{\dd \mu}, \dot \mu \rangle
        =  - \alpha \| \nabla \left(1- {\dd \pi}/{\dd \mu}\right) \|^2_{L^2(\mu)} - \beta\| 1- {\dd \pi}/{\dd \mu} \|^2_{L^2(\mu)}
        \\
        \leq  - \beta\| 1- {\dd \pi}/{\dd \mu} \|^2_{L^2(\mu)}
        .
    \end{multline*}
    By the functional inequality \eqref{eq:Loj-FR} in Theorem~\ref{thm:exponential-decay-of-inclusive-KL-divergence}, we obtain the decay result for the inclusive KL functional.
\end{proof}

\begin{proof}
    [Proof of Proposition~\ref{prop:local-estimator-wgf-spokoiny}]
    The linear local-estimator formula is classical, see e.g., \citep[Section 4.6]{spokoiny2016nonparametric}.
    We exploit the force-kernelization formula~\eqref{eq:force-kernelization-of-wasserstein-gradient-flow}, specifically,
    \begin{align}
        \mu  \dFdmut(x)
        =
        \mu - \pi
        \text{ for }
        F(\mu) = \DKL(\pi|\cdot)
        \text{ (i.e., inclusive KL)}
    \end{align}
    The key is to show that the quantities required for the solution of nonparametric regression~\eqref{eq:local-lin-wls-velocity} can be estimated using samples from the gradient flows. We denote $\xi = \dFdmut$ and
the population version of $S_0(x)$ as 
\begin{multline}
    S^\infty_0(x) 
    :=
    \frac{1}{\int\mu K}
    \int\mu\xi K
    =
    \frac{1}{\int\mu K}
    \int \mu_t(x')  \xi(x') K(x' - x) \dd  x'
    \\
    =
    1
    -
    \frac{1}{\int\mu K}
    \int  \pi (x') K(x' - x) \dd  x'
    \approx
    1
    -
    \frac{N}{M}
    \frac{\sum_{i=1}^M K(z_i-x)}
    {\sum_{l=1}^N K(x_l-x)}
    =S_0(x)
\end{multline}
which is a Monte-Carlo estimator using samples of $\mu$ and $\pi$ as given in \eqref{eq:spokoiny-response-moments}. The other quantities for the regression solution can be similarly estimated. Hence we obtain the statement of Proposition~\ref{prop:local-estimator-wgf-spokoiny}.
\end{proof}

\section{Derivation for Fisher-Rao Gaussian and Bures-Wasserstein flows}
\label{sec:FRBW-WGF}

\paragraph{Fisher-Rao Gaussian gradient flow (Proposition~\ref{prop:gvi-gfe})}
The inclusive KL divergence can be written as
\[
\KL(\pi | \rho) = \int \pi(x) \log {\pi(x)}
-\int \pi(x) \log \rho(x)
\, dx = H(\pi | \rho) - H(\pi),
\]
where the cross-entropy is given by
$H(\pi | \rho) = -\int \pi(x) \log \rho(x) \, dx$
and $H(\pi):= - \int \pi(x) \log \pi(x) \, dx$ is the entropy of $\pi$.

Since the entropy term $H(\pi)$ does not involve $\rho$,
we only need to consider the cross-entropy term $H(\pi | \rho)$ when taking derivatives w.r.t. parameters of $\rho$.
A direct calculation yields the derivative formula
(noting that $\rho = N(\mu, \Sigma)$):
\[
\nabla_\mu H(\pi | \rho) = \Sigma^{-1} (\mu - m_\pi),
\quad
    \nabla_\Sigma H(\pi | \rho)
    =
    \frac{1}{2} \Sigma^{-1}\left(  
     \Sigma_\pi - \Sigma
      + 
       \left(m_\pi - \mu\right)\left(m_\pi - \mu\right)^T
     \right)
    \Sigma^{-1}
    ,
\]
with the notation $m_\pi, \Sigma_\pi$ defined as in \eqref{eq:gvi-gfe-notation} for $\pi$ that is not necessarily Gaussian.

To use the above formula to derive the gradient flow equation in the Fisher-Rao and Bures-Wasserstein geometry,
we use the gradient structure for the Gaussian manifold derived
by \citet{lieroEvolutionGaussiansHellingerKantorovichBoltzmann2025}, namely
the following explicit formula:

\begin{theorem}
    [Corollary of \citet{lieroEvolutionGaussiansHellingerKantorovichBoltzmann2025}]
    The Gaussian gradient flow equation is given by
    the system of ODEs:
    \begin{align}
        \dot m &= - \bbK_i^\mathrm{\text{mean}} \nabla_m \KL(\pi | \rho) \\
        \dot \Sigma &= - \bbK_i^\mathrm{\text{cov}} \nabla_\Sigma \KL(\pi | \rho)
    \end{align}
    where $i\in \{\OT, \FR\}$ for the Bures-Wasserstein and Fisher-Rao Onsager operator:
    \begin{align*}
        \bbKotto^\mathrm{\text{mean}}(m, \Sigma) x &= x, &
        \bbKotto^\mathrm{\text{cov}}(m, \Sigma) X &= 2\left(X\Sigma + \Sigma X\right), \\
        \bbK_\FR^\mathrm{\text{mean}}(m, \Sigma) x &= \Sigma x, &
        \bbK_\FR^\mathrm{\text{cov}}(m, \Sigma) X &= 2\Sigma X \Sigma.
    \end{align*}
\end{theorem}
Using the derivative formula and the Onsager operator, we directly obtain the Fisher-Rao Gaussian gradient flow equation in Proposition~\ref{prop:gvi-gfe}.

\begin{proof}
    [Proof of Proposition~\ref{cor:fr-ode-solution}]
    We observe that the mean equation in \eqref{eq:gvi-gfe} is decoupled from the covariance variable.
    Hence, 
    the solution to the mean equation is obvious:
    $$
    m_t = (1 - e^{-t}) m_\pi + e^{-t} m_0.
    $$
    Plugging the mean solution into the covariance ODE, we have
    $$
    \dot \Sigma_t 
    =
    -\left(   {\Sigma_t} - \Sigma_\pi - 
    e^{-2t} (m_0 - m_\pi)(m_0 - m_\pi)^T
    \right)
    .
    $$
    Multiplying both sides by $e^t$, we have
    \begin{align}
    \label{eq:covariance-ode-intermediate}
    e^t \dot \Sigma_t 
    =
    -e^t \left(   {\Sigma_t} - \Sigma_\pi - 
    e^{-2t} (m_0 - m_\pi)(m_0 - m_\pi)^T
    \right)
    .
    \end{align}
    Using the product rule,
    \begin{align}
    \label{eq:product-rule-covariance}
    \frac{d}{dt} (e^t \Sigma_t)
    =
    e^t \dot \Sigma_t + e^t \Sigma_t
    .
    \end{align}
Combining \eqref{eq:covariance-ode-intermediate} and \eqref{eq:product-rule-covariance}, we have
\begin{align}
    \frac{d}{dt} (e^t \Sigma_t)
    =
    e^t   \Sigma_\pi + 
    e^{-t} (m_0 - m_\pi)(m_0 - m_\pi)^T
    .
    \label{eq:covariance-ode-solution}
\end{align}
Integrating both sides, we have
\begin{align}
    e^t \Sigma_t
    =
    \Sigma_0 +
    (e^t - 1) \Sigma_\pi + 
    (1 - e^{-t}) (m_0 - m_\pi)(m_0 - m_\pi)^T
    .
\end{align}
Dividing both sides by $e^t$, we obtain the desired result.
\end{proof}

\section{Further background on Wasserstein gradient flows}
\label{sec:wasserstein-gradient-flows-appendix}
We provide further background on Wasserstein gradient flows, especially on the pseudo-Riemannian structure of the Wasserstein space.

The Onsager operator, as well as the Riemannian metric tensor $\mathbb G_\OT = \bbKotto^{-1}$, induces a duality pairing between the tangent and cotangent spaces.    
We use the unweighted space for simplicity.
Note that the calculation can also be made in the weighted space $L^2(\rho)$.
    \begin{align}
        \text{duality pairing: }
        {}_{\text{dual}}\langle \xi, \bbKotto(\rho) \ \zeta \rangle_{\text{primal}}
        =
        \langle {\xi}, \bbKotto(\rho) \ \zeta \rangle_{L^2}
        =
        \int \xi \cdot \bbKotto(\rho) \ \zeta
        .
        \label{eq:duality-pairing-wasserstein}
    \end{align}

The Stein geometry can also be characterized
in this way.
\citet{duncan2019geometry}
proposed the following Onsager operator that is a modification of the Otto's Wasserstein formalism,
\begin{align}
    \bbK_\textrm{Stein}(\rho): 
    T^*_\rho \calM \to T_\rho \calM, \xi \mapsto
    -\DIV(\rho\cdot \ID \circ \Tkrho \nabla \xi)
    .
    \label{eq:onsager-stein}
\end{align}
The resulting Stein gradient flow equation is given by
  \begin{align*}
    \partial_t \mu 
    = - \bbK_\textrm{Stein}(\mu) \log \frac{\dd\mu}{\dd\pi}
    =\DIV \left(\mu {\K_\mu}\nabla \left(V + \log\mu\right) \right) 
    .
\end{align*}

We now look at the Wasserstein gradient flow of the inclusive KL divergence.
    The gradient flow equation can be given by the Otto's formal calculation,
    $$
    \nabla _\OT \DKL(\pi \| \mu) =\  
     \bbKotto \partial \DKL(\pi \| \mu)
    = - \DIV\left(\mu \nabla \left(1-\frac{\dd \pi}{\dd \mu}\right)\right)
    $$
    where $\bbKotto$ is the Wasserstein Onsager operator,
    \ie the inverse of the Riemannian metric tensor $\mathbb G_\OT$ of the Wasserstein manifold.

        A standard characterization of the Wasserstein gradient flow is the following
        energy dissipation equality
        in the inclusive KL setting
        \begin{align}
            \DKL(\pi | \mu_t) - \DKL(\pi | \mu_s)
            =
            - \int_s^t
            \left\|
            \nabla \left(1-\frac{\dd \pi}{\dd \mu_r}\right)
            \right\|_{L^2(\mu_r)}^2 \dd r
            .
            \label{eq:ede-rkl-wgf}
        \end{align}
        The dissipation of the inclusive KL divergence energy,
        a.k.a. the production of the relative entropy,
        equals
        the integral of the Sobolev norm of the differential of the inclusive KL along the curve $\mu_r$.
For completeness, we provide a standard characterization via the following differential energy dissipation equality
\begin{multline}
    \frac{\mathrm{d}}{\mathrm{d}t} \DKL(\pi | \mu_t)
    =
    \langle 1 - \frac{\dd \pi}{\dd \mu}, \dot \mu_t \rangle
    =
    \langle 1 - \frac{\dd \pi}{\dd \mu}, \bbKotto \partial 
    \DKL(\pi | \mu)\rangle
    \\
    =
    \langle 1 - \frac{\dd \pi}{\dd \mu}, \DIV(\mu \nabla  \left(1-\frac{\dd \pi}{\dd \mu}\right))\rangle
    =
    -\bigg\|
    \nabla \left(1-\frac{\dd \pi}{\dd \mu}\right)
    \bigg\|_{L^2(\mu)}^2
    .   
\end{multline}
Integrating both sides,
the integral form of EDE is then given by
\eqref{eq:ede-rkl-wgf}.

\section{Further background on Fisher-Rao/Hellinger gradient flows}
We provide further background on Fisher-Rao and Hellinger gradient flows, especially on the technicalities of the Hellinger flows over positive measures $\Mplus$.

We first consider the Hellinger flow of the exclusive (reverse) KL divergence over the positive measures $\Mplus$.
Its gradient flow equation, the reaction equation, is given by
        \begin{align}
            \dot \mu = - \mu\cdot  \log \frac{\dd \mu}{\dd \pi}
            .
            \label{eq:fr-gfe}
        \end{align}
The gradient structure is given by
    \begin{align}
        \begin{cases}
            \textrm{{Space} :}& 
            \text{positive measures } \Mplus  
            \\
            \textrm{
                 Energy functional} :& {
                     \text{exclusive KL: } \DKL( \cdot | \pi)}
            \\
            \textrm{
                 Dissipation Geometry} :& {
                    \text{Hellinger}}
        \end{cases}
        \label{eq:gradient-structure-fkl-he}
    \end{align}

        One can further restrict the gradient flow to the probability measures by modifying the dynamics in \eqref{eq:fr-gfe} with a projection onto the probability measures,
        \ie
        \begin{align}
            \dot \mu = - \mu\cdot\left(
                \log \frac{\dd \mu}{\dd \pi}
                -
            \int \log \frac{\dd \mu}{\dd \pi} \dd \mu
                \right)
            .
            \label{eq:fr-gfe-prob}
        \end{align}
        The resulting ODE is the gradient flow equation over the Fisher-Rao manifold of the probability measures, also known as the spherical Hellinger manifold~\citep{LasMie19GPCA,mielke2025hellinger}.
        That is, it has the following gradient structure:
        \begin{align}
            \begin{cases}
                \textrm{{Space} :}& 
                \text{probability measures } \calP 
                \\
                \textrm{ Energy functional} :& {\text{exclusive KL: } \DKL( \cdot | \pi)}
                \\
                \textrm{ Dissipation Geometry} :& { \text{spherical Hellinger a.k.a. Fisher-Rao}}
            \end{cases}
            \label{eq:gradient-structure-revKL-SHe-gf-pure}
        \end{align}

For the inclusive (forward) KL divergence,
as discussed in the main text,
the Hellinger flow over the positive measures $\Mplus$ coincides with the Fisher-Rao flow over the probability measures $\calP$, given the same initial condition.
Specifically,
the Hellinger gradient structure over the positive measures $\Mplus$ is given by
    \begin{align}
        \begin{cases}
            \textrm{{Space} :}& 
            \text{positive measures } \Mplus 
            \\
            \textrm{Energy functional: inclusive}& {F(\cdot):= \DKL(\pi | \cdot)}
            \\
            \textrm{Dissipation Geometry} :& { \text{Hellinger}}
        \end{cases}
        \label{eq:gradient-structure-revKL-he-gf-pure}
    \end{align}
    This flow actually contains Fisher-Rao flow over probability
    if initialized as probability measures.

For the Hellinger flows,
an interesting and known analysis result is that the following Polyak-\Loj functional inequality cannot hold globally for
    exclusive KL divergence functional:
    \begin{align}
        \biggl\|\log \frac{\dd\mu}{\dd \pi}\biggr\|^2_{L^2_{\mu}}
        \geq 
        c\cdot \operatorname{\mathrm{D}_\mathrm{KL}}(\mu\|\pi)
        \text{ for all } \mu \in \Mplus
        .
        \label{eq:local-KL-LSI-Loj}
    \end{align}
    Inequality~\eqref{eq:local-KL-LSI-Loj} differs from the typical log-Sobolev inequality
    in that no Sobolev norm is involved.
    As a consequence of the works of \citet{mielke2025hellinger,carrillo2024fisher}, we obtain the following lemma regarding the property of the Hellinger flows of the exclusive KL. This is in sharp contrast to the case of the inclusive KL as discussed in the main text.
    \begin{lemma}
        [No global PL condition in Hellinger or Fisher-Rao flows of KL]
        There
        exists no $c>0$ such that \eqref{eq:local-KL-LSI-Loj} holds
        for all positive measures $\mu\in \Mplus$.
        \label{lm:local-vs-global-Loj}
    \end{lemma}
    We note that
    the same (negative) result holds for
    the exclusive-KL-Fisher-Rao gradient flow over the probability measures $\calP$.

\section{Other implications on applications}
\label{sec:implications}
    \subsection{Generative modeling}
    \label{sec:gen-model}
    There has been a surge of interest in formulating GANs in the fashion of Wasserstein gradient flows.
    Promising
    empirical results have been reported by
    \citet{ansariRefiningDeepGenerative2021,yiMonoFlowRethinkingDivergence2023,yiBridgingGapVariational2023,franceschiUnifyingGansScorebased2024,hengGenerativeModelingFlowGuided2024}.
    In addition, there have also been
    a series of paper that present theoretical analysis of
    GAN training dynamics as
    \emph{interacting gradient flows} by,
    \eg
    \citet{hsiehFindingMixedNash2018,domingo-enrichMeanfieldAnalysisTwoplayer2020,wangExponentiallyConvergingParticle2022,wangLocalConvergenceGradient2023,pmlr-v238-dvurechensky24a}.
    Using this paper's insight, we now uncover connections between generative models and the Wasserstein gradient flow of the inclusive KL functional.
    
    Our starting point is the standard divergence-based generative modeling training, which solves the optimization problem
    \begin{equation}
        \min_{\theta} \DKL(\pi_\textrm{data} | {g_\theta}_{\#}P_Z)
        ,
        \label{eq:gan-inclusive-KL-minimization-main}
    \end{equation}
    where $P_Z$ is the latent variable distribution, \eg standard Gaussian, and ${g_\theta}_\#$ is the
    push-forward operation
    using a
    generator network $g_\theta$.
    Following our force-kernelized WGF framework as in \eqref{eq:kernelized-gfe-reverseKL}, consider a force-kernelized projected gradient flow for the inclusive KL minimization~\eqref{eq:gan-inclusive-KL-minimization-main}.
    \begin{align}
        \dot \theta
        = - \Pi_\Theta \left( - \DIV \left( \mu_\theta \nabla \calT_{K, \mu_\theta}\left(1 -  \frac{d\pi}{d\mu_\theta}\right)  \right) \right),\quad {\mu_\theta = {g_\theta}_{\#}P_Z}
        \label{eq:theta-pde-inclusive-KL-wgf-main}
    \end{align}
    where
    $\Pi_\Theta$ is the projection (of the Riemannian gradient) onto the parameter space $\Theta$.
    From \eqref{eq:theta-pde-inclusive-KL-wgf-main},
    we immediately observe that
    the flow can be written as
    \begin{align}
        \dot \theta
        = - \Pi_\Theta \left( - \DIV \left( \mu_\theta
        \nabla f^*(x)
        \right)
        \right)
        , \quad
        f^*(x)
        =
        \int  K(x', x) \left(\mu_\theta(x') -  {d\pi}(x')\right)
        \dd x'
        .
        \label{eq:theta-pde-inclusive-KL-wgf-main-opt-test}
    \end{align}
    Discretizing the dynamics, we have
    \begin{align}
        \theta^{l+1} \gets \theta^l - \eta^l \Pi_\Theta \left(
            - \DIV \left( \mu_\theta
        \nabla f^*(x)
        \right)
        \right)
    \end{align}
    which corresponds to the training dynamics of MMD-GANs
    \citep{liGenerativeMomentMatching2015,dziugaiteTrainingGenerativeNeural2015,liMmdGanDeeper2017,binkowskiDemystifyingMmdGans2018}
    using the optimal test function $f^*$.
    The insight of our paper is that,
    through the lens of \eqref{eq:theta-pde-inclusive-KL-wgf-main},
    we can view the MMD-GAN type generative model training dynamics as
    performing inclusive KL inference.
    
\subsection{\citet{chewiSVGDKernelizedWasserstein2020}'s sampler based on kernelized WGF of $\chi^2$-divergence}
    Previously,
    \citet{chewiSVGDKernelizedWasserstein2020} proposed a kernelized Wasserstein gradient flow of the $\chi^2$-divergence.
    They considered the kernelized gradient flow equation
    $\displaystyle \dot \mu =   \DIV\left(\mu \K_\mu \nabla \frac{\dd \mu}{\dd \pi} \right)$,
        where, technically, $\K_\mu $ should be taken as the integral operator defined by
        $\K_\mu f = \ID \circ \Tkmu$.
        However, in their implemented algorithm, they switched the order of the operators $\nabla$ and $\K_\mu$, either as a heuristic or practical means.
        That is, what they actually implemented (in \citep[Section~4]{chewiSVGDKernelizedWasserstein2020}) is
        \begin{equation}
            \dot \mu = \DIV\left(\mu \nabla \K_\pi  \frac{\dd \mu}{\dd \pi} \right)
            .
            \label{eq:chewi-actually-did}
        \end{equation}
    
        From this paper's perspective,
        this is kernelizing the generalized dual-force in our definition, rather than the velocity function $\nabla \left(\frac{\dd \mu}{\dd \pi}- 1 \right)$.
    We can now derive a force-kernelized WGF of the $\chi^2$-divergence from the first principle.
    Consider the gradient flow equation
    \begin{align}
        \dot \mu =   \DIV\left(\mu \nabla \calT_{K, \pi} \left(\frac{\dd \mu}{\dd \pi} - 1\right)\right)
        .
        \label{eq:chi-square-flow}
    \end{align}
    Note that the integral operator $\calT_{K, \pi}$ is associated with the target measure $\pi$, rather than the measure $\mu$ as in \eqref{eq:kernelized-gfe-reverseKL}.
    Nonetheless, a simple observation is that \eqref{eq:chi-square-flow} formally coincides with \eqref{eq:kernelized-gfe-reverseKL}.
    Therefore, we conclude that a \emph{principled force-kernelized flow of the $\chi^2$-divergence WGF \eqref{eq:chi-square-flow} is equivalent to the WGF of the MMD studied by \citep{arbel_maximum_2019,korbaKernelSteinDiscrepancy2021}}, which is straightforward to implement and in contrast to using the ad-hoc scheme of \citep{chewiSVGDKernelizedWasserstein2020}.

    \subsection{Kernel Stein discrepancy descent as inclusive KL minimization}
    \label{sec:ksd-as-ikl}
    The original implementation of \eqref{eq:wgf-mmd-pde}
    by \citet{arbel_maximum_2019}
    suffers from a few drawbacks such as mode collapse or slow convergence.
    Instead of the MMD,
    authors such as
    \citet{korbaKernelSteinDiscrepancy2021,chenSteinPoints2018,barp2019minimum}
    advocated for minimizing the kernel Stein discrepancy (KSD)~\citep{gorham2017measuring,liuKernelizedSteinDiscrepancy2016,chwialkowskiKernelTestGoodness2016}
    for inference.
    From the optimization perspective,
    we replace the MMD objective with the KSD objective
    $ \frac12\ksd^2(\mu | \pi)$.
    The KSD can be viewed
     as a special case of the MMD associated with the Stein kernel~\citep{gorham2017measuring,liuKernelizedSteinDiscrepancy2016,chwialkowskiKernelTestGoodness2016}.
    The Wasserstein gradient flow equation of the KSD can be straightforwardly calculated as noted by \citet{korbaKernelSteinDiscrepancy2021}.
    \begin{align}
        \dot \mu = \DIV(\mu \cdot \int\nabla_2 \spi(x, \cdot ) \dd \mu(x))
        ,
        \tag{KSD-WGF}
        \label{eq:wgf-ksd-pde}
    \end{align}
    where $\spi$ is the Stein kernel.
    Unlike \eqref{eq:wgf-mmd-pde},
    to implement a discrete-time algorithm that simulates \eqref{eq:wgf-ksd-pde},
    we only need to evaluate the score function $\nabla \log \pi$ without needing the samples from $\pi$.
    As KSD can be viewed as a special case of MMD with the Stein kernel,
    using our characterization of the MMD-WGF in Theorem~\ref{thm:equivalence-of-gradient-flow-equations},
    we obtain the following insight
    that views the KSD minimization also as inclusive KL inference:
    \begin{corollary}
      [Formal equivalence between KSD-WGF and inclusive KL inference]
      The WGF equation of KSD~\eqref{eq:wgf-ksd-pde} is equivalent to
      \eqref{eq:kernelized-gfe-reverseKL},
      which is the kernelized WGF of the inclusive KL divergence energy functional when the Stein kernel $\spi$ is used.
      \label{cor:equivalence-of-gradient-flow-equations-ksd}
    \end{corollary}

    \subsection{Discrete-time algorithms for inference and sampling via kernelized Wasserstein flows}
    Much of this works is concerned with the continuous-time perspective.
    For completeness, we now discuss the discretized gradient flows and their practical implications.
    Suppose our goal is to construct a computational algorithm to approximate a target distribution $\pi$ via the inclusive KL minimization~\eqref{eq:inclusiveKL}.
    We consider two settings:
    \textbf{(1)} we have access to samples from the target $y^i\sim \pi$, e.g., in generative modeling;
    \textbf{(2)} we have access to the score function $\nabla \log \pi$, e.g., in inference and sampling.
    Our scheme is based on discretizing the force-kernelized Wasserstein gradient flow equation~\eqref{eq:kernelized-gfe-reverseKL}, obtaining the discrete-time update scheme
    \begin{equation}
        X_{t+1} = X_t - \tau
        \nabla
        \int \nabla_2 K(x', x ) \frac{\delta F}{\delta \mu}[\mu_t] (x') \, \dd \mu_t(x')
        .
        \label{eq:kernelized-wgf}
    \end{equation}
    An interacting particle system
    can be
    simulated by considering
    particle approximation to the measure,
    $\mu = \frac1n \sum_{i=1}^{n}\delta_{x_i},\ x_i\in\bbR^d$.
    \paragraph{Setting (1): sample-based setting with flows of MMD}
    In general, for Wasserstein gradient flow of the energy functional $F$,
    one may implement a practical algorithm
    that discretizes the PDE~\eqref{eq:wasserstein-gfe}.
    As discussed in the beginning of Section~\ref{sec:gradient-flows},
    in the vanilla Wasserstein gradient flow of the inclusive KL divergence \eqref{eq:vanilla-wasserstein-rkl-gfe},
    the velocity field $\nabla  \left(1- \frac{\dd \pi}{\dd \mu_t}(X_t)
        \right) $
        cannot be implemented out of the box.
        Based on Theorem~\ref{thm:equivalence-of-gradient-flow-equations},
        we now resort to \eqref{eq:kernelized-wgf} which we have shown to be algorithmically equivalent to \citet{arbel_maximum_2019}'s algorithm which they termed MMD-flow.
        This amounts to simulating (in discrete time) an interacting particle system:
    \begin{equation}
        X_{t+1}^i = X_t^i - \tau
        \left(
            \frac{1}{N} \sum_{j=1}^N \nabla _2 K(X_t^j, X_t^i)
            -
            \frac{1}{M} \sum_{j=1}^M \nabla _2 K(Y_t^j, X_t^i)
            \right)
        ,
        \label{eq:kmv-gd-disc-time-sample-mmd}
    \end{equation}
    where $X_t^i$ are samples from the distribution $\mu_t$;
    cf. \citep{arbel_maximum_2019}
    for the experimental results.
    
    \paragraph{Setting (2): score-based setting with flows of KSD}
    In variational inference and sampling, we typically have access to the target $\pi$ in the form of the score function $\nabla \log \pi$ without samples.
    Discretizing the PDE~\eqref{eq:wgf-ksd-pde}, we have
    \begin{equation}
        X_{t+1} = X_t - \tau
        \int \nabla _2 \spi(x, \cdot) \dd \mu(X_t)
        .
        \label{eq:kmv-gd-disc-time-ksd}
    \end{equation}
    A sample-based implementation of the above algorithm is then given by
    \begin{equation}
        X_{t+1}^i = X_t^i - \tau
        \left(
            \frac{1}{N} \sum_{j=1}^N \nabla _2 \spi(X_t^j, X_t^i)
            \right)
        .
        \label{eq:kmv-gd-disc-time-sample-ksd}
    \end{equation}

    \subsection{Continuous optimization: discrete-time mirror descent}
    \label{sec:mirror-descent}
    There have many studies using mirror descent of the exclusive KL divergence such as \citep{karimiSinkhornFlowMirror2024,chopinConnectionTemperingEntropic2024,aubin-frankowskiMirrorDescentRelative2022}.
    We now provide the details of inclusive KL minimization  via mirror descent.
    Consider an explicit Euler scheme
    \begin{align}
        \min _{\rho \in \calP} \langle \partial \DKL(\pi | \rho), \rho \rangle + \frac1\tau \DKL( \rho | \rho^l)
        .
    \end{align}
    where $\langle, \rangle$ is the $L^2$ inner product.
    Using the
    optimality condition of this optimization problem, we can derive the following mirror descent update:
    \begin{align}
        \rho^{l+1} (x) \gets
        \frac1{Z^l}
        \rho^l (x) \cdot \exp \left( - \tau
            \left(
                1 - \frac{\dd \pi}{\dd \rho^l}
            \right)
        \right)
        \text{ for all } x \in \bbR^d
        ,
    \end{align}
    where $Z^l$ is the normalization constant.
    Other than simple finite-dimensional cases in the optimization literature,
    this update is infinite-dimensional and not implementable in practice.
    We now apply
    the kernel approximation of this paper to obtain
    kernelized mirror descent updates:
    \begin{multline*}
        \rho^{l+1} (x) \gets
        \frac1{Z^l}
        \rho^l (x) \cdot \exp \left( - \tau
        \cdot
        \calT_{K, \rho^l}
            \left(
                1 - \frac{\dd \pi}{\dd \rho^l}
            \right)
        \right)
        \\
        =
        \frac1{Z^l}
        \rho^l (x) \cdot \exp \left( - \tau
        \cdot
        \int K(x, y)
        \left(
            \rho^l(y) - \pi(y)
            \right)
            \dd y
        \right)
        .
    \end{multline*}
    
    Similarly, using Stein's method,
    we can also perform update only via the score function of the target $\nabla \log \pi$.
    \begin{align}
        \rho^{l+1}  \gets
        \frac1{Z^l}
        \rho^l  \cdot \exp \left( - \tau
            \cdot
            \int \spi (x, y) \rho^l(y) \dd y
        \right)
        .
    \end{align}
    We refer to recent works such as
    \citep{lazicPropagationChaosFisherRao2026,zhu2024kernel}
    for detailed discussions on
    kernelized Fisher-Rao gradient flows.

\section{Further numerical results}
We now test the behavior of $\WGFloc$ under different bandwidths.
We measure the final discrepancy by the energy distance, which is proportional to the MMD associated with the distance-induced kernel
$k(x,y)=\frac12\bigl(\|x-x_0\|+\|y-x_0\|-\|x-y\|\bigr)$
for a fixed base point $x_0$.
Figure~\ref{fig:local-mmd-robustness} reports two bandwidth sweeps.  In the top row, the local-estimator flow is run with Gaussian, Laplace, inverse-multiquadric, and Student-$t$ kernels.  The final error remains low over a wide range of bandwidths.
The Laplace and Student-$t$ kernels are nearly bandwidth-invariant on the tested interval $h\in[0.6,5]$, while the Gaussian kernel eventually loses locality as $h$ grows.
The bottom row of Figure~\ref{fig:local-mmd-robustness} compares the local-estimator flow with the vanilla MMD-WGF of \citet{arbel_maximum_2019} across Gaussian-kernel bandwidths.  The vanilla MMD-WGF exhibits a pronounced U-shaped sensitivity curve, whereas the local-estimator flow maintains low error over a substantially wider bandwidth range.  The improvement is most visible on the anisotropic banana target.
\begin{figure}[htbp]
    \centering
    \includegraphics[width=\textwidth]{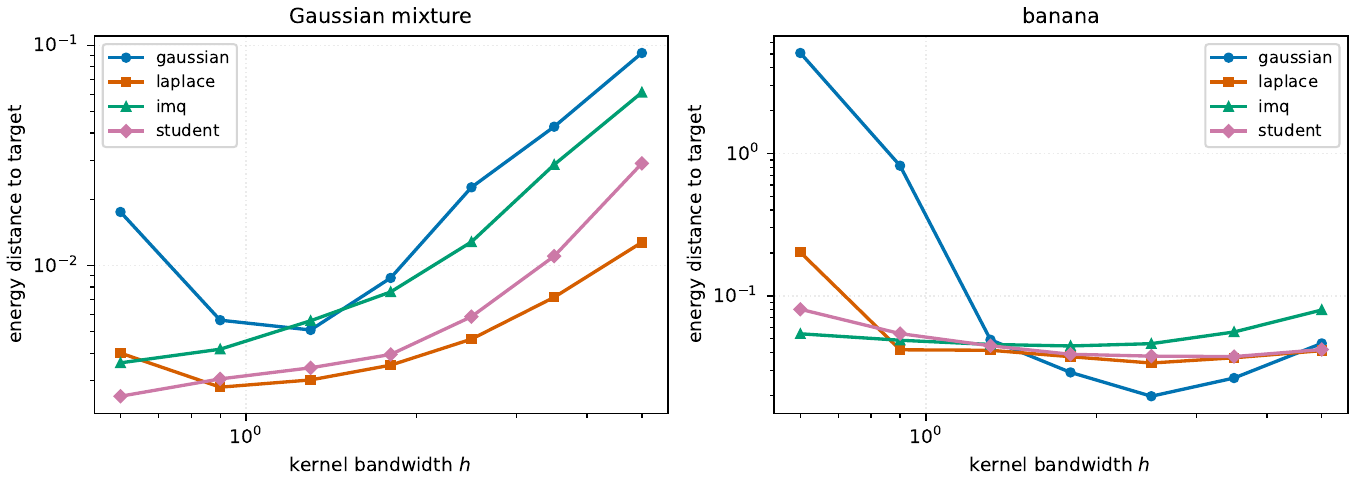}\\[3pt]
    \includegraphics[width=\textwidth]{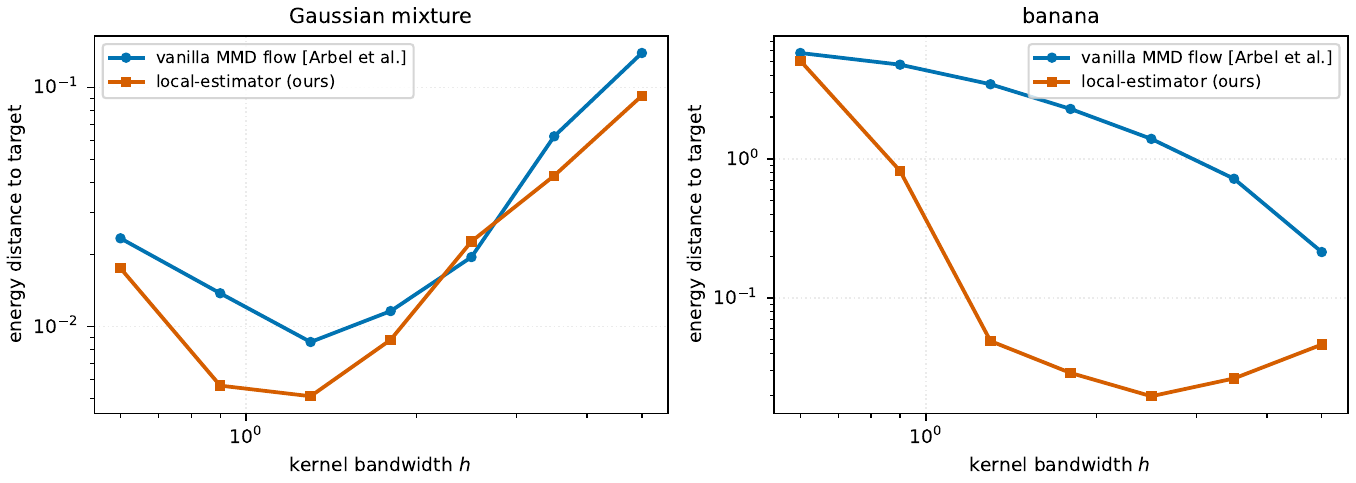}
    \caption{Bandwidth robustness, measured by final energy distance to the target.  Top: local-estimator flow with four kernels on the Gaussian-mixture target (left) and banana target (right).  Bottom: local-estimator flow versus the vanilla MMD flow of \citet{arbel_maximum_2019} across Gaussian-kernel bandwidths.}
    \label{fig:local-mmd-robustness}
\end{figure}

\clearpage

\section*{Acknowledgements}
The author thanks
Bharath Sriperumbudur,
Wittawat Jitkrittum,
Thomas M\"ollenhoff
for the helpful discussions on related topics.
The author also thanks Vladimir Spokoiny and Alexander Mielke,
whose works and suggestions greatly inspired this work.
For the numerical simulations,
Anthropic Claude was used for assistance with the code development.
The author assumes responsibility for all content.
The author has no relevant financial or non-financial interests to disclose.
\small
\bibliographystyle{siamplain}
\bibliography{ref}
\end{document}